\documentclass[11pt]{article}

\usepackage[linesnumbered,ruled,vlined]{algorithm2e}
\usepackage{amsfonts}
\usepackage{amsthm}
\usepackage{geometry}
\usepackage{hyperref}
\usepackage{multirow}
\usepackage[nocomma]{optidef}
\usepackage{setspace}
\usepackage{siunitx}
\usepackage{thmtools}
\usepackage{natbib}
\bibpunct[, ]{(}{)}{,}{a}{}{,}

\newcommand{\bias}{b}
\newcommand{\biasb}{\mathbf{\bias}}
\newcommand{\cl}{\mathrm{cl}}
\newcommand{\class}{t}
\newcommand{\classbar}{\bar{\class}}
\newcommand{\classh}{\hat{\class}}
\newcommand{\constc}{c}
\newcommand{\conv}{\mathrm{conv}}
\newcommand{\failcount}{\mathrm{fail\_count}}
\newcommand{\funca}{a}
\newcommand{\funcf}{f}
\newcommand{\gentwovarineq}{\mathrm{GenTwoVarIneq}}
\newcommand{\layer}{\ell}
\newcommand{\layercount}{L}
\newcommand{\lb}{\mathrm{LB}}
\newcommand{\maxfail}{\mathrm{max\_fail}}
\newcommand{\maxiter}{\mathrm{max\_iter}}
\newcommand{\neuroncount}{n}
\newcommand{\neuroni}{i}
\newcommand{\neuronj}{j}
\newcommand{\neuronjs}{J}
\newcommand{\neuronk}{k}
\newcommand{\neurons}{N}
\newcommand{\neuronsf}{\neurons_{\mathrm{fix}}}
\newcommand{\neuronsp}{\neurons_{\mathrm{pair}}}
\newcommand{\nn}{\mathbb{N}}
\newcommand{\norm}[1]{\lVert#1\rVert_\normp}
\newcommand{\normp}{p}
\newcommand{\objupperbound}{\bar{z}_{\mathrm{UB}}}
\newcommand{\one}{\mathtt{1}}
\newcommand{\ones}{\mathbf{1}}
\newcommand{\optobjbound}{\bar{z}}
\newcommand{\optobj}{z^*}
\newcommand{\optobjlp}{z_{\mathrm{LP}}^*}
\newcommand{\pert}{\epsilon}
\newcommand{\pertinit}{\pert_{\mathrm{init}}}
\newcommand{\pertl}{\pert_{\mathrm{LB}}}
\newcommand{\pertu}{\pert_{\mathrm{UB}}}
\newcommand{\phaseonefixvar}{\mathrm{PhaseOneFixVar}}
\newcommand{\phasetwofixvar}{\mathrm{PhaseTwoFixVar}}
\newcommand{\quant}{q}
\newcommand{\refer}{R}
\newcommand{\rr}{\mathbb{R}}
\newcommand{\setx}{X}
\newcommand{\setxf}{\setx_{\mathrm{feas}}}
\newcommand{\setxi}{\setx_{\mathrm{in}}}
\newcommand{\setxo}{\setx_{\mathrm{out}}}
\newcommand{\superth}{^{\mathrm{th}}}
\newcommand{\ub}{\mathrm{UB}}
\newcommand{\updateinapprox}{\mathrm{UpdateInApprox}}
\newcommand{\varu}{u}
\newcommand{\varub}{\mathbf{\varu}}
\ifdefined\varv
    \renewcommand{\varv}{v}
\else
    \newcommand{\varv}{v}
\fi
\newcommand{\varvb}{\mathbf{\varv}}
\newcommand{\varvbh}{\hat{\varvb}}
\newcommand{\varvh}{\hat{\varv}}
\newcommand{\varx}{x}
\newcommand{\varxb}{\mathbf{\varx}}
\newcommand{\varxbar}{\bar{\varx}}
\newcommand{\varxbbar}{\bar{\varxb}}
\newcommand{\varxbh}{\hat{\varxb}}
\newcommand{\varxh}{\hat{\varx}}
\ifdefined\vary
    \renewcommand{\vary}{y}
\else
    \newcommand{\vary}{y}
\fi
\newcommand{\varyb}{\mathbf{\vary}}
\newcommand{\varybh}{\hat{\varyb}}
\newcommand{\varz}{z}
\newcommand{\varzb}{\mathbf{\varz}}
\newcommand{\varzbh}{\hat{\varzb}}
\newcommand{\varzh}{\hat{\varz}}
\newcommand{\verifybnn}{\mathrm{VerifyBnn}}
\newcommand{\weight}{W}
\newcommand{\weightb}{\mathbf{\weight}}
\newcommand{\weightl}{w}
\newcommand{\zeros}{\mathbf{0}}
\newcommand{\zz}{\mathbb{Z}}
\ifdefined\Halmos
\else
    \newcommand{\Halmos}{\mbox{\quad$\square$}}
\fi
\ifdefined\SingleSpacedXI
\else
    \newcommand{\SingleSpacedXI}{}
\fi
\ifdefined\proofname
    \renewcommand{\proofname}{\hskip-\labelsep\spacefactor3000}
\else
\fi
\ifdefined\qedsymbol
    \renewcommand{\qedsymbol}{}
\else
\fi

\SetKw{Continue}{continue}
\SetKw{Break}{break}

\allowdisplaybreaks

\newtheorem{theorem}{Theorem}
\newtheorem{lemma}[theorem]{Lemma}

\DeclareMathOperator*{\argmax}{arg\,max} 

\title{On Integer Programming for the Binarized Neural Network Verification Problem}
\author{Woojin Kim, James Luedtke}

\begin{document}

\maketitle

\begin{abstract}
    Binarized neural networks (BNNs) are feedforward neural networks with binary weights and activation functions. In the context of using a BNN for classification, the verification problem seeks to determine whether a small perturbation of a given input can lead it to be misclassified by the BNN, and the robustness of the BNN can be measured by solving the verification problem over multiple inputs. The BNN verification problem can be formulated as an integer programming (IP) problem. However, the natural IP formulation is often challenging to solve due to a large integrality gap induced by big-$M$ constraints. We present two techniques to improve the IP formulation. First, we introduce a new method for obtaining a linear objective for the multi-class setting. Second, we introduce a new technique for generating valid inequalities for the IP formulation that exploits the recursive structure of BNNs. We find that our techniques enable verifying BNNs against a higher range of input perturbation than existing IP approaches within a limited time.
\end{abstract}

\section{Introduction}\label{sec:1}

Binarized neural networks (BNNs) are feedforward neural networks with binary weights and activation functions (\cite{hubara2016binarized}). BNNs consist of the input layer $\layer=0$, $\layercount$ hidden layers $\layer=1,\ldots,\layercount$, and the output layer $\layer=\layercount+1$. Each layer $\layer\in \{0,1,\ldots,L+1\}$ consists of $\neuroncount^\ell$ neurons, $\neurons^\ell = \{1, \ldots, \neuroncount^\ell \}$. Every $\neuroni \in \neurons^\layer$ is connected to every $\neuronj \in \neurons^{\layer - 1}$ with weight $\weight_{\neuroni \neuronj}^\layer \in \{-1, 0, 1\}$ and has bias $\bias_\neuroni^\layer \in \zz$ for $\layer \in \{1, \ldots, \layercount + 1\}$.

Thanks to binary weights and activation functions, BNNs drastically reduce memory size and accesses and substantially improve power-efficiency over standard neural networks, replacing most arithmetic operations with bit-wise operations (\cite{hubara2016binarized}). BNNs have also achieved nearly state-of-the-art results in image classification (\cite{hubara2016binarized}). Moreover, BNNs have achieved results comparable to feedforward neural networks in image detection (\cite{kung2018efficient}), image super resolution (\cite{ma2019efficient}), and text classification (\cite{shridhar2020end}). For these reasons, BNNs have been applied in small embedded devices (\cite{mcdanel2017embedded}).

In this work, we study the verification problem associated with a given BNN that is used to classify feature vectors. BNNs map a feature vector in $[0, 1]^{\neuroncount^0}$ whose coordinates are quantized as multiples of $\frac{1}{\quant}$ for $\quant \in \nn$ to the real output vector by the function $\funcf = (\funcf_1, \ldots, \funcf_{\neuroncount^{\layercount + 1}}): \frac{1}{\quant} \zz_+^{\neuroncount^0} \cap [0, 1]^{\neuroncount^0} \to \rr^{\neuroncount^{\layercount + 1}}$ which is defined recursively using weights and biases in each layer. (See Section \ref{sec:2} for the detailed definition of $\funcf$.) Each $\class \in \neurons^{\layercount + 1}$ corresponds to a class. A feature vector $\varxbbar \in \frac{1}{\quant} \zz_+^{\neuroncount^0} \cap [0, 1]^{\neuroncount^0}$ is classified as the class $\classbar \in \neurons^{\layercount + 1}$ that corresponds to $\argmax_{\class \in \neurons^{\layercount + 1}} \funcf_\class(\varxbbar)$.

Verifying BNNs against input perturbation for multiple feature vectors provides a measure of their robustness. The BNN verification problem is defined for a given trained BNN, a given feature vector $\varxbbar$ with class $\classbar$, and input perturbation $\pert > 0$. We say that $\varxbbar$ is $\pert$-verified if the BNN classifies $\varxb^0$ as $\classbar$ for every feature vector $\varxb^0 \in \frac{1}{\quant} \zz_+^{\neuroncount^0} \cap [0, 1]^{\neuroncount^0}$ having distance between $\varxb^0$ and $\varxbbar$ at most $\pert$. The BNN verification problem is to determine whether a given $\varxbbar$ can be $\pert$-verified in the BNN. The answer is true if and only if there does not exist a feature vector $\varxb^0$ with distance at most $\pert$ from $\varxbbar$ satisfying $\funcf_\class(\varxb^0) > \funcf_{\classbar}(\varxb^0)$ for some class $\class \ne \classbar$, which is equivalent to the following maximum $\optobj_\pert(\varxbbar)$ being non-positive:
\begin{equation}
    \optobj_\pert(\varxbbar) = \max_{\substack{\varxb^0 \in \frac{1}{\quant} \zz_+^{\neuroncount^0} \cap [0, 1]^{\neuroncount^0}\\ \norm{\varxb^0 - \varxbbar} \le \pert}} \Big\{\funcf_\class(\varxb^0) - \funcf_{\classbar}(\varxb^0): \class \in \neurons^{\layercount + 1} \setminus \{\classbar\}\Big\}.\label{eq:bnnVeri}
\end{equation}
Here, the maximum perturbation from $\varxbbar$ is defined using an $\ell_\normp$ norm. For the most part, in this work we focus on the case of $\normp = 1$, but the integer programming (IP) approaches we consider can also be directly applied to the cases of $\normp = \infty$ and $\normp = 2$, with the caveat that when using an $\ell_2$ norm, the resulting IP formulations we consider become integer (convex) quadratic programs, as opposed to integer linear programs.

This paper contributes to the literature exploring IP methods to solve the BNN verification problem. IP methods to solve the BNN verification problem can be implemented by solving \eqref{eq:bnnVeri} to optimality -- we refer to this as the BNN verification optimization problem. However, to verify a BNN, the process of solving \eqref{eq:bnnVeri} can be terminated as soon as the sign of $\optobj_\pert(\varxbbar)$ is determined. Specifically, the solution process of \eqref{eq:bnnVeri} can be terminated when a feasible solution with a positive objective value is found, or a non-positive upper bound for the optimal objective value is obtained. If we terminate an optimization algorithm for solving \eqref{eq:bnnVeri} according to this condition, we refer to this as the BNN verification problem.

Our first contribution is to present a new way for obtaining a linear objective of the BNN verification problem by creating a single optimization problem that incorporates the decision of which alternative class a perturbed feature vector is misclassified as, rather than considering each alternative class separately. Our second contribution is to describe a new technique for generating valid inequalities for the IP formulation by exploiting the recursive structure of BNNs. These valid inequalities, called layerwise derived valid inequalities, are generated by considering one layer at a time and solving IP subproblems to check validity among a set of natural candidate inequalities. While deriving these valid inequalities requires solving IP subproblems, these IP subproblems are much easier to solve than the full BNN verification problem because they consider only a single layer at a time and do not include any of the ``big-$M$'' constraints required for the IP formulation of the BNN verification problem.

\cite{khalil2018combinatorial} study the problem of attacking a BNN, which is equivalent to the BNN verification problem. They present a mixed-integer linear programming (MILP) formulation of the problem and propose a heuristic for generating solutions that works by propagating a target from the last layer to the first layer by solving a single-layer MILP at each step. This approach is proven to be successful at finding successful attacks (i.e., showing an input $\bar{x}$ cannot be $\epsilon$-verified) but cannot positively verify a solution. Interestingly, our method also proceeds in layers, but from the initial layer forwards, and its strength is in being able to $\epsilon$-verify inputs.
\cite{han2021single} also study the MILP formulation of the BNN verification problem and study the convex hull for a single neuron to improve this MILP formulation. \cite{lubczyk2024neuron} extended Han and G{\'o}mez's work by exploring the convex hull for a pair of neurons in the same hidden layer. Using this convex hull, they proposed a constraint generation framework to solve the MILP problem for the BNN verification problem. In related work on feedforward (not binarized) neural networks, \cite{fischetti2018deep} addressed the problem of verifying such networks. They formulated this problem as an MILP problem and exploited the idea of fixing variables to solve this MILP problem. Our layerwise derived valid inequalities represent an adaptation and extension of this approach to the BNN setting.

Another approach for solving the BNN verification problem is to formulate it as a Boolean satisfiability problem (\cite{amir2021smt,jia2020efficient,kovasznai2021portfolio,narodytska2018verifying}). These papers formulated the condition that $\varxbbar$ can be $\pert$-verified in a given BNN with Boolean formulas in the case that the maximum perturbation from $\varxbbar$ is defined using an $\ell_\infty$ norm. \cite{ivashchenko2023verifying} explored a different method for solving the BNN verification problem by describing the set of all feasible output vectors or an overapproximation of this set as a star set and checking whether this star set contains an output vector from a perturbed feature vector misclassified by the BNN. This work also handles the case that the maximum perturbation is defined using an $\ell_\infty$ norm. \cite{shih2019verifying} and \cite{zhang2021bdd4bnn} addressed a different problem of {\it counting} the number of perturbed feature vectors misclassified by the BNN in the case that the maximum perturbation is defined using an $\ell_1$ norm. We focus on the IP approach due to its inherent flexibility, for example, in easily adapting the set of allowed perturbations from $\varxbbar$ to be defined by an $\ell_1$, $\ell_2$, or $\ell_\infty$ norm, and to investigate and extend the limits of the IP approach.

While our work focuses on verifying a given BNN after it has been trained, for completeness, we briefly mention some methods for training BNNs. \cite{hubara2016binarized} trained BNNs by using a gradient descent method similar to the training of feedforward neural networks, and \cite{geiger2020larq} trained BNNs by using a pseudo-gradient method. \cite{bernardelli2023bemi} and \cite{toro2019training} trained BNNs by solving constraint programming or mixed integer programming problems. To improve training BNNs, \cite{martinez2020training} applied activation scaling, and \cite{tang2017train} applied activation approximation.

Finally, we mention that the BNN verification problem is closely related to the problem of finding counterfactual explanations for such networks. Counterfactual explanations seek to find a feature vector that is as close as possible to a given feature vector while leading to a specific alternative desired prediction (\cite{mothilal2020explaining,wachter2018counterfactual}). Recently, counterfactual explanations for several machine learning prediction models, including k-nearest neighbors, have been explored (\cite{contardo2024optimal,vivier2024cf}). A method for solving the BNN verification problem can be used to solve the counterfactual explanation problem by conducting a binary search to (approximately) find the smallest $\epsilon$ such that the feature is classified as a target class.

This paper is organized as follows. In Section \ref{sec:2}, we introduce an existing IP formulation of the BNN verification problem and describe methods for obtaining linear constraints and a linear objective in this IP formulation, including our method for obtaining a linear objective by incorporating the decision of which alternative class a perturbed feature vector is misclassified into a single optimization problem. In Section \ref{sec:3}, we describe our technique for deriving layerwise derived valid inequalities and our approach for using these valid inequalities to solve the BNN verification problem. We present the results of a computational study in Section \ref{sec:4} and conclude with a discussion of potential future directions in Section \ref{sec:5}.

\textbf{Notation.} We use $[M]$ to represent the set $\{1, \ldots, M\}$ for $M \in \nn$.
\section{IP Formulations}\label{sec:2}

To obtain an IP formulation for problem \eqref{eq:bnnVeri}, $\varxb^0 \in \frac{1}{\quant} \zz_+^{\neuroncount^0} \cap [0, 1]^{\neuroncount^0}$ is defined as a vector of the decision variables that represent a perturbed feature vector, and $\varxb^\layer \in \{0, 1\}^{\neuroncount^\layer}$ is defined as the vector of decision variables that represent the output of the $\layer\superth$ hidden layer for $\layer \in [\layercount]$, which is obtained recursively from $\varxb^{\layer - 1}$ as explained in the following paragraphs. Also, $\setx^0$ is defined as the set of all  perturbed feature vectors to verify over:
$$
    \setx^0 = \Big\{\varxb^0 \in \frac{1}{\quant} \zz_+^{\neuroncount^0} \cap [0, 1]^{\neuroncount^0}: \norm{\varxb^0 - \varxbbar} \le \pert \Big\}.
$$

The usual description of a BNN describes the output of each neuron as $+1$ or $-1$. In order to derive an IP formulation that can be directly solved by IP solvers, we will instead encode the outputs with binary ($0/1$) valued variables. Given a binary output vector $\varxb^{\layer-1}$, it can be converted to a $+1/-1$ valued vector via the transformation $2 \varxb^{\layer-1} - \ones$. Thus, for each layer $\layer \in [\layercount + 1]$, the first component of BNN propagation is to multiply the input $2 \varxb^{\layer - 1} - \ones$ with $\weightb^\layer$  and add $\biasb^\layer$. We introduce the notation $\funca^\layer = (\funca_1^\layer, \ldots, \funca_{\neuroncount^\layer}^\layer)$ for this affine function of  $\varxb^{\layer - 1}$:
$$
    \funca^\layer(\varxb^{\layer - 1}) = \weightb^\layer(2 \varxb^{\layer - 1} - \ones) + \biasb^\layer.
$$
Then, the output of each each $\layer \in [\layercount]$ is obtained from its input $\varxb^{\layer - 1}$ via the transformation
$$
    \varxb^\layer = \one_{\rr_+}(\funca^\layer(\varxb^{\layer - 1})),
$$
where $\one_{\rr_+}$ is the indicator function mapping non-negative numbers to $1$ and negative numbers to $0$.\footnote{This is equivalent to the standard description of BNN propagation in which the $+1/-1$ output is defined by the sign function.}

Applying $\funca^\layer$ and $\one_{\rr_+}$ alternately $\layercount$ times and $\funca^{\layercount + 1}$ lastly results in $\funcf$. Thus, $\funcf$ can be written as follows:
$$
    \funcf(\varxb^0) = \funca^{\layercount + 1}(\one_{\rr_+}(\funca^\layercount(\one_{\rr_+}(\funca^{\layercount - 1}( \cdots \one_{\rr_+}(\funca^1(\varxb^0)) \cdots ))))).
$$

Using these observations, \eqref{eq:bnnVeri} can be written as the following (nonlinear) IP problem:
\begin{maxi!}|s|
    {\varxb^0, \ldots, \varxb^\layercount}
    {\max\Big\{\funca_\class^{\layercount + 1}(\varxb^\layercount) - \funca_{\classbar}^{\layercount + 1}(\varxb^\layercount): \class \in \neurons^{\layercount + 1} \setminus \{\classbar\}\Big\}\label{eq:origFormObj}}
    {\label{eq:origForm}}
    {}
    \addConstraint{}{\varxb^\layer = \one_{\rr_+}(\funca^\layer(\varxb^{\layer - 1})),}{\quad \forall \layer \in [\layercount],\label{eq:origFormPropConstr}}
    \addConstraint{}{\varxb^0 \in \setx^0,\quad\label{eq:origFormXInpVar}}
    \addConstraint{}{\varxb^\layer \in \{0, 1\}^{\neuroncount^\layer},}{\quad \forall \layer \in [\layercount].\label{eq:origFormXHidVar}}
\end{maxi!}

\subsection{Constraint Formulation}\label{subsec:2_1}

The layer propagation constraints \eqref{eq:origFormPropConstr} and the input perturbation constraint \eqref{eq:origFormXInpVar} include nonlinear constraints in their description. We review how these constraints can be reformulated using linear constraints in order to obtain a linear IP formulation.

\subsubsection{Layer Propagation Constraint Formulation}\label{subsubsec:2_1_1}

To formulate the layer propagation constraints \eqref{eq:origFormPropConstr} with linear constraints, the following constants are defined for $\layer \in [\layercount]$ and $\neuroni \in \neurons^\layer$:
\begin{equation}
    \begin{aligned}
        &\lb_\neuroni^\layer := \sum_{\neuronj \in \neurons^{\layer - 1}} (\weight_{\neuroni \neuronj}^\layer - |\weight_{\neuroni \neuronj}^\layer|)\\
        &\ub_\neuroni^\layer := \sum_{\neuronj \in \neurons^{\layer - 1}} (\weight_{\neuroni \neuronj}^\layer + |\weight_{\neuroni \neuronj}^\layer|)\\
        &\refer_\neuroni^\layer := \begin{cases} \frac{2}{\quant} \bigg\lceil\frac{\quant(\sum_{\neuronj \in \neurons^0} \weight_{\neuroni \neuronj}^1 - \bias_\neuroni^1)}{2}\bigg\rceil - \frac{1}{\quant} & (\layer = 1)\\ 2 \bigg\lceil\frac{\sum_{\neuronj \in \neurons^{\layer - 1}} \weight_{\neuroni \neuronj}^\layer - \bias_\neuroni^\layer}{2}\bigg\rceil - 1 & (\layer \in \{2, \ldots, \layercount\}) .\end{cases}.
    \end{aligned}
\label{eq:const}
\end{equation}
Since $\weight_{\neuroni \neuronj}^\layer \in \{-1, 0, 1\}$ it follows that for any $\varxb^{\layer - 1} \in [0, 1]^{\neuroncount^{\layer - 1}}$ 
\[ 
\lb_\neuroni^\layer
\leq 2 \sum_{\neuronj \in \neurons^{\layer - 1}} \weight_{\neuroni \neuronj}^\layer \varx_\neuronj^{\layer - 1}
\leq \ub_\neuroni^\layer . \]
The following lemma describes how the layer propagation constraints can be modeled with linear constraints.
This extends the formulation of \cite{han2021single}, which assumes $\weight_{\neuroni \neuronj}^\layer \in \{-1, 1\}$ and $\bias_\neuroni^\layer = 0$.  The proof is in the Appendix.

\begin{restatable}{lemma}{linpropconstr}\label{lem:linPropConstr}
    Each $(\varxb^0, \varxb^1) \in \setx^0 \times \{0, 1\}^{\neuroncount^1}$  satisfies \eqref{eq:origFormPropConstr} for $\layer=1$ if and only if it satisfies the following inequalities:
\begin{align}
        &2 \sum_{\neuronj \in \neurons^0} \weight_{\neuroni \neuronj}^1 \varx_\neuronj^0 \ge \Big(\refer_\neuroni^1 - \lb_\neuroni^1 + \frac{1}{\quant}\Big)\varx_\neuroni^1 + \lb_\neuroni^1, \quad \forall \neuroni \in \neurons^1,\label{eq:propInpLowerBoundConstr}\\
        &2 \sum_{\neuronj \in \neurons^0} \weight_{\neuroni \neuronj}^1 \varx_\neuronj^0 \le \Big(\ub_\neuroni^1 - \refer_\neuroni^1 + \frac{1}{\quant}\Big)\varx_\neuroni^1 + \Big(\refer_\neuroni^1 - \frac{1}{\quant}\Big), \quad \forall \neuroni \in \neurons^1.\label{eq:propInpUpperBoundConstr}
\end{align}
    For $\layer \in \{2, \ldots, \layercount\}$, $(\varxb^{\layer - 1}, \varxb^\layer) \in \{0, 1\}^{\neuroncount^{\layer - 1}} \times \{0, 1\}^{\neuroncount^\layer}$ satisfies \eqref{eq:origFormPropConstr} if and only if it satisfies the following inequalities:
    \begin{align}
        &2 \sum_{\neuronj \in \neurons^{\layer - 1}} \weight_{\neuroni \neuronj}^\layer \varx_\neuronj^{\layer - 1} \ge (\refer_\neuroni^\layer - \lb_\neuroni^\layer + 1)\varx_\neuroni^\layer + \lb_\neuroni^\layer, \quad \forall \neuroni \in \neurons^\layer,\label{eq:propHidLowerBoundConstr}\\
        &2 \sum_{\neuronj \in \neurons^{\layer - 1}} \weight_{\neuroni \neuronj}^\layer \varx_\neuronj^{\layer - 1} \le (\ub_\neuroni^\layer - \refer_\neuroni^\layer + 1)\varx_\neuroni^\layer + (\refer_\neuroni^\layer - 1), \quad \forall \neuroni \in \neurons^\layer.\label{eq:propHidUpperBoundConstr}
    \end{align}
\end{restatable}
Thus, the nonlinear constraints \eqref{eq:origFormPropConstr} can be replaced 
with the linear constraints \eqref{eq:propInpLowerBoundConstr}- \eqref{eq:propHidUpperBoundConstr}.

\subsubsection{Input Perturbation Constraint Formulation}\label{subsubsec:2_1_2}

The input perturbation constraint \eqref{eq:origFormXInpVar}
can be formulated by defining the integer-valued decision variable $\vary_\neuroni$ to represent $\quant \varx_\neuroni^0$, for $\neuroni \in \neurons^0$. Then, $\varxb^0 \in \setx^0$ if and only if $\varxb^0 \in \rr_+^{\neuroncount^0}$, and there exist $\varyb \in \zz_+^{\neuroncount^0}$ satisfying the following constraints:
\begin{align}
    &\varyb = \quant \varxb^0,\label{eq:inpQuantConstr}\\
    &\varyb \le \quant \ones,\label{eq:inpUpperBoundConstr}\\
    &\norm{\varxb^0 - \varxbbar} \le \pert.\label{eq:inpPertConstr}
\end{align}
In the case of $\normp = 1$ or $\normp = \infty$, \eqref{eq:inpPertConstr} can be formulated with linear constraints using standard techniques. In the case of $\normp = 2$, \eqref{eq:inpPertConstr} can be formulated as a convex quadratic constraint.

\subsection{Objective Formulation}\label{subsec:2_2}

Finally, to obtain a linear IP formulation of \eqref{eq:origForm} we explore two methods to obtain a linear objective.

The first approach, which has been used in \cite{han2021single,lubczyk2024neuron}, is to consider separately each alternative class $\class \in \neurons^{\layercount + 1} \setminus \{\classbar\}$. Each such class $\class$ results in the following IP problem with linear objective:
\begin{maxi}|s|
    {}
    {\funca_\class^{\layercount + 1}(\varxb^\layercount) - \funca_{\classbar}^{\layercount + 1}(\varxb^\layercount)}
    {\label{eq:indivForm}}
    {\optobj_\pert(\varxbbar, \class) =}
    \addConstraint{}{\eqref{eq:propInpLowerBoundConstr}-\eqref{eq:propInpUpperBoundConstr},}
    \addConstraint{}{\eqref{eq:propHidLowerBoundConstr}-\eqref{eq:propHidUpperBoundConstr},}{\quad \forall \layer \in \{2, \ldots, \layercount\},}
    \addConstraint{}{\eqref{eq:inpQuantConstr}-\eqref{eq:inpPertConstr},}
    \addConstraint{}{\varxb^0 \in \rr_+^{\neuroncount^0},}
    \addConstraint{}{\varxb^\layer \in \{0, 1\}^{\neuroncount^\layer},}{\quad \forall \layer \in [\layercount],}
    \addConstraint{}{\varyb \in \zz_+^{\neuroncount^0}.}
\end{maxi}
The optimal solution of \eqref{eq:origForm} can be obtained by solving \eqref{eq:indivForm} for each $\class \in \neurons^{\layercount + 1} \setminus \{\classbar\}$ and choosing $\class$ that achieves the maximum of $\optobj_\pert(\varxbbar, \class)$.

For problems with more than two classes, we propose an alternative approach for obtaining a linear objective which directly includes the choice of the alternative class $t$ in the formulation, thereby avoiding the need to solve \eqref{eq:indivForm} for each  
$\class \in \neurons^{\layercount + 1} \setminus \{\classbar\}$. Thus, for each $\class \in \neurons^{\layercount + 1} \setminus \{\classbar\}$ let $\varz_\class \in \{0, 1\}$ be a decision variable indicating whether $\class$ is selected as the alternative class, and for each $\neuroni \in \neurons^\layercount$ let $\varv_{\class \neuroni} \in \{0, 1\}$ be a decision variable that is used to represent the product $\varz_\class \varx_\neuroni^\layercount$. The proof of the following lemma is in the appendix.

\begin{lemma}
\label{lem:oneip}
    The following IP problem is equivalent to \eqref{eq:origForm}:
    \begin{maxi!}|s|
        {}
        {\sum_{\class \in \neurons^{\layercount + 1} \setminus \{\classbar\}} \sum_{\neuroni \in \neurons^\layercount} 2(\weight_{\class \neuroni}^\layercount - \weight_{\classbar \neuroni}^\layercount)\varv_{\class \neuroni}\nonumber}
        {\label{eq:incorpForm}}
        {}
        \breakObjective{+ \sum_{\class \in \neurons^{\layercount + 1} \setminus \{\classbar\}} \Big(-\sum_{\neuroni \in \neurons^\layercount} (\weight_{\class \neuroni}^\layercount - \weight_{\classbar \neuroni}^\layercount) + (\bias_\class^\layercount - \bias_{\classbar}^\layercount)\Big)\varz_\class\label{eq:incorpFormObj}}
        \addConstraint{}{\eqref{eq:propInpLowerBoundConstr}-\eqref{eq:propInpUpperBoundConstr}\nonumber,}
        \addConstraint{}{\eqref{eq:propHidLowerBoundConstr}-\eqref{eq:propHidUpperBoundConstr}\nonumber,}{\quad \forall \layer \in \{2, \ldots, \layercount\},}
        \addConstraint{}{\eqref{eq:inpQuantConstr}-\eqref{eq:inpPertConstr},\nonumber}
        \addConstraint{}{\sum_{\class \in \neurons^{\layercount + 1} \setminus \{\classbar\}} \varz_\class = 1,\label{eq:incorpFormClassConstr}}
        \addConstraint{}{\sum_{\class \in \neurons^{\layercount + 1} \setminus \{\classbar\}} \varv_{\class \neuroni} = \varx_\neuroni^\layercount,}{\quad \forall \neuroni \in \neurons^\layercount,\label{eq:incorpFormProdSumConstr}}
        \addConstraint{}{\varv_{\class \neuroni} \le \varz_\class,}{\quad \forall \class \in \neurons^{\layercount + 1} \setminus \{\classbar\},\ \neuroni \in \neurons^\layercount,\label{eq:incorpFormProdConstr}}
        \addConstraint{}{\varxb^0 \in \rr_+^{\neuroncount^0},\label{eq:incorpFormXInpVar}}
        \addConstraint{}{\varxb^\layer \in \{0, 1\}^{\neuroncount^\layer},}{\quad \forall \layer \in [\layercount],\label{eq:incorpFormXHidVar}}
        \addConstraint{}{\varyb \in \zz_+^{\neuroncount^0},\label{eq:incorpFormYVar}}
        \addConstraint{}{\varz_\class \in \{0, 1\},}{\quad \forall \class \in \neurons^{\layercount + 1} \setminus \{\classbar\},\label{eq:incorpFormZVar}}
        \addConstraint{}{\varv_{\class \neuroni} \in \{0, 1\},}{\quad \forall \class \in \neurons^{\layercount + 1} \setminus \{\classbar\},\ \neuroni \in \neurons^\layercount.\label{eq:incorpFormVVar}}
    \end{maxi!}
\end{lemma}

\subsection{Valid Inequalities for Neurons}\label{subsec:2_3}

\cite{han2021single} explored the following convex hull for a single neuron for $\layer \in \{2, \ldots, \layercount\}$ and $\neuroni \in \neurons^\layer$ in the case of $\biasb^\layer = \zeros$:
\begin{equation}
    \conv(\{(\varxb^{\layer - 1}, \varx_\neuroni^\layer) \in \{0, 1\}^{\neuroncount^{\layer - 1}} \times \{0, 1\}: \varx_\neuroni^\layer = \one_{\rr_+}(\funca_\neuroni^\layer(\varxb^{\layer - 1}))\}).
\label{eq:singleConvHull}
\end{equation}

In the following theorem, we extend the results in \cite{han2021single} to a general case where $\biasb^\layer$ are not necessarily zero. This theorem is proved in the Appendix.

\begin{theorem}\label{theo:singleConvHullExt}
    For $\layer \in \{2, \ldots, \layercount\}$ and $\neuroni \in \neurons^\layer$, \eqref{eq:singleConvHull} is the set of $(\varxb^{\layer - 1}, \varx_\neuroni^\layer) \in [0, 1]^{\neuroncount^{\layer - 1}} \times [0, 1]$ satisfying the following inequalities for $\neuronjs \subset \{\neuronj \in \neurons^{\layer - 1}: \weight_{\neuroni \neuronj}^\layer \ne 0\}$:
    \begin{subequations}\label{eq:singleConvHullIneq}
        \begin{align}
            &\sum_{\neuronj \in \neuronjs} (\weight_{\neuroni \neuronj}^\layer(2 \varx_\neuronj^{\layer - 1} - 1) - (2 \varx_\neuroni^\layer - 1)) \ge (\refer_\neuroni^\layer - \ub_\neuroni^\layer + 1)\varx_\neuroni^\layer ,\label{eq:singleConvHullLowerBoundIneq}\\
            &\sum_{\neuronj \in \neuronjs} (\weight_{\neuroni \neuronj}^\layer(2 \varx_\neuronj^{\layer - 1} - 1) - (2 \varx_\neuroni^\layer - 1)) \le (\refer_\neuroni^\layer - \lb_\neuroni^\layer - 1)(-\varx_\neuroni^\layer + 1).\label{eq:singleConvHullUpperBoundIneq}
        \end{align}
    \end{subequations}
\end{theorem}
Since the number of inequalities described in this theorem grows exponentially with the number of nodes in a hidden layer, these would be used to improve the LP relaxation of the formulation \eqref{eq:incorpForm} by adding them via a cutting plane procedure.

\cite{lubczyk2024neuron} extended \cite{han2021single} by investigating the following convex hull for a pair of neurons in the same hidden layer for $\layer \in \{2, \ldots, \layercount\}$ and $\neuroni, \neuronk \in \neurons^\layer$ satisfying $\neuroni < \neuronk$ in the case where $\neuroncount^{\layer - 1}$ is an even integer larger than $2$, every entry in $\weightb^\layer$ is nonzero, and $\biasb^\layer = 0$:
\[
    \conv(\{(\varxb^{\layer - 1}, \varx_\neuroni^\layer, \varx_\neuronk^\layer): \varxb^{\layer - 1} \in \{0, 1\}^{\neuroncount^{\layer - 1}}, \varx_\neuroni^\layer = \one_{\rr_+}(\funca_\neuroni^\layer(\varxb^{\layer - 1})), \varx_\neuronk^\layer = \one_{\rr_+}(\funca_\neuronk^\layer(\varxb^{\layer - 1}))\}).
\]
In this work, a relaxation IP problem to the IP problem \eqref{eq:indivForm} is solved for a fixed alternative class $\class \in \neurons^{\layercount + 1}$ to tackle the BNN verification problem. This computational study shows some improvement in the optimal objective value of the LP relaxation of the IP formulation compared to \cite{han2021single}, but this does translate to significant improvement in the ability to solve the BNN verification problem.
\section{Layerwise Derived Valid Inequalities}\label{sec:3}

We explore an alternative approach for deriving valid inequalities for \eqref{eq:indivForm} that is based on recursively approximating the set of ``reachable vectors'' at each layer.

Starting from $\setx^0$, $\setx^\layer$ is defined recursively as follows for $\layer \in [\layercount]$:
$$
    \setx^\layer = \one_{\rr_+}(\funca^\layer(\setx^{\layer - 1})).
$$
Then, $\setx^\layer$ is the set of all output vectors of layer $\layer$ that can be obtained from a perturbed input vector in the set $X^0$. Since the objective \eqref{eq:origFormObj} only depends on $\varxb^\layercount$, \eqref{eq:origForm} can be formulated as the following IP problem:
\begin{maxi}|s|
    {\varxb^\layercount}
    {\max\Big\{\funca_\class^{\layercount + 1}(\varxb^\layercount) - \funca_{\classbar}^{\layercount + 1}(\varxb^\layercount): \class \in \neurons^{\layercount + 1} \setminus \{\classbar\}\Big\}}
    {}
    {}
    \addConstraint{\varxb^\layercount \in \setx^\layercount.}
\end{maxi}
The above IP problem implies that \eqref{eq:origForm} can be solved with access to a description for $\setx^\layercount$, which motivates studying valid inequalities for this set. We observe that valid inequalities for $\setx^\layer$ can be obtained using a relaxation of the set $\setx^{\layer-1}$ since  $\setx^\layer = \one_{\rr_+}(\funca^\layer(\setx^{\layer - 1}))$. Since a full description for $\setx^0$ is given, this motivates an iterative approach in which we derive valid inequalities for the set $X^{\ell}$ based on an outer approximation of the set $X^{\ell-1}$ for $\ell=2,\ldots,L$. We refer to such inequalities as \textit{layerwise derived valid inequalities}.

\subsection{Variable Fixing}\label{subsec:3_1}

We first investigate layerwise derived valid inequalities for $\setx^\layer$ for $\layer \in [\layercount]$ of the form
\begin{equation}
    \varx_\neuroni^\layer = \frac{1-\constc_\neuroni}{2} \label{eq:varFix}
\end{equation}
for $\neuroni \in \neurons^\layer$ and $\constc_\neuroni \in \{-1, 1\}$.
In the case of $\constc_\neuroni = -1$, \eqref{eq:varFix} is equivalent to $\varx_\neuroni^\layer = 1$ and \eqref{eq:varFix} is equivalent to $\varx_\neuroni^\layer = 0$ in the case of $\constc_\neuroni = 1$. For these reasons, valid equalities for $\setx^\layer$ of the form \eqref{eq:varFix} are called variable fixings. This idea of identifying variable fixings is related to the work of \cite{fischetti2018deep} who apply a similar technique to fix binary variables in the IP formulation of the feedforward neural network verification problem.

Our approach is to consider each equality of the form \eqref{eq:varFix} and determine if it is a valid equality. The following theorem describes an IP problem that can be solved to determine validity of such an equality.

\begin{theorem}\label{theo:varFix}
    Consider $\layer \in [\layercount]$, $\neuroni \in \neurons^\layer$, and $\constc_\neuroni \in \{-1, 1\}$.  Let $\setxo^{\layer - 1} \subset \{0, 1\}^{\neuroncount^{\layer - 1}}$ satisfy $\setxo^{\layer - 1} \supseteq \setx^{\layer - 1}$ and define
    \begin{equation}
        \optobj = \max\Big\{\constc_\neuroni \sum_{\neuronj \in \neurons^{\layer - 1}} \weight_{\neuroni \neuronj}^\layer \varx_\neuronj^{\layer - 1}: \varxb^{\layer - 1} \in \setxo^{\layer - 1}\Big\} .\label{eq:varFixForm}
    \end{equation}
    If $\optobj \le \constc_\neuroni \refer_\neuroni^\layer/2$, where $\refer_\neuroni^\layer$ is defined in \eqref{eq:const},
    then \eqref{eq:varFix} is a valid equality for $\setx^{\layer}$.
\end{theorem}

The details are described in Section \ref{subsec:3_3}, but to preview this, the idea is to use Theorem \ref{theo:varFix} to obtain outer approximations of the sets $X^\ell$ by setting $\setxo^0$ as $\setx^0$, and then recursively using the theorem to derive valid inequalities to define valid inequalities to define $\setxo^{\layer}$ based on the previously determined $\setxo^{\layer-1}$ for $\layer=1,\ldots,\layercount$.

To prove this theorem, we use the following lemma.
\begin{lemma}
\label{lem:propEqui}
    Consider $\layer \in [\layercount]$, $\neuroni \in \neurons^\layer$, and $\constc_\neuroni \in \{-1, 1\}$. Then $\varxb^{\layer-1} \in \{0, 1\}^{\neuroncount^{\layer - 1}}$ satisfies $ \constc_\neuroni \sum_{\neuronj \in \neurons^{\layer - 1}} \weight_{\neuroni \neuronj}^\layer \varx_\neuronj^{\layer - 1} \le \constc_\neuroni \refer_\neuroni^\layer/2$ if and only if $\one_{\rr_+}(\funca_\neuroni^\layer(\varxb^{\layer - 1})) = \frac{1-\constc_\neuroni}{2}$.
\end{lemma}

\begin{proof}{Proof.}
    In the case of $\constc_\neuroni = -1$, the following statements hold for $\layer = 1$:
\begin{equation*}
    \begin{aligned}
        &2 \constc_\neuroni \sum_{\neuronj \in \neurons^0} \weight_{\neuroni \neuronj}^1 \varx_{\neuronj}^0 \le \constc_\neuroni \refer_\neuroni^1,&\\
        \Leftrightarrow\quad &2 \sum_{\neuronj \in \neurons^0} \weight_{\neuroni \neuronj}^1 \varx_{\neuronj}^0 \ge \refer_\neuroni^1,&\\
        \Leftrightarrow\quad &2 \sum_{\neuronj \in \neurons^0} \weight_{\neuroni \neuronj}^1 \varx_{\neuronj}^0 \ge \refer_\neuroni^1 + \frac{1}{q},\quad &(2 \sum_{\neuronj \in \neurons^0} \weight_{\neuroni \neuronj}^1 \varx_\neuronj^0, \refer_\neuroni^1 + \frac{1}{q} \in \frac{2}{q}\zz)\\
        \Leftrightarrow\quad &2 \sum_{\neuronj \in \neurons^0} \weight_{\neuroni \neuronj}^1 \varx_{\neuronj}^0 \ge \sum_{\neuronj \in \neurons^0} \weight_{\neuroni \neuronj}^1 - \bias_\neuroni^1,\quad &(\text{$\refer_\neuroni^1 + \frac{1}{q}$ is the smallest element of $\frac{2}{q} \zz$}\\
        &&\text{not smaller than $\sum_{\neuronj \in \neurons^0} \weight_{\neuroni \neuronj}^1 - \bias_\neuroni^1$})\\
        \Leftrightarrow\quad &\funca_\neuroni^1(\varxb^0) \ge 0,&\\
        \Leftrightarrow\quad &\one_{\rr_+}(\funca_\neuroni^1(\varxb^0)) = 1 = \frac{1-\constc_\neuroni}{2}.&
    \end{aligned}
    \end{equation*}
    For $\layer \in \{2, \ldots, \layercount\}$, the same arguments hold because $2 \sum_{\neuronj \in \neurons^{\layer - 1}} \weight_{\neuroni \neuronj}^\layer \varx_\neuronj^{\layer - 1}$ and $\refer_\neuroni^\layer + 1$ are even integers, and $\refer_\neuroni^\layer + 1$ is the smallest even integer not smaller than $\sum_{\neuronj \in \neurons^{\layer - 1}} \weight_{\neuroni \neuronj}^\layer - \bias_\neuroni^\layer$.

    In the case of $\constc_\neuroni = 1$, the following statements hold for $\layer = 1$:
  \begin{equation*}
    \begin{aligned}
        &2 \constc_\neuroni \sum_{\neuronj \in \neurons^0} \weight_{\neuroni \neuronj}^1 \varx_{\neuronj}^0 \le \constc_\neuroni \refer_\neuroni^1,&\\
        \Leftrightarrow\quad &2 \sum_{\neuronj \in \neurons^0} \weight_{\neuroni \neuronj}^1 \varx_{\neuronj}^0 \le \refer_\neuroni^1,&\\
        \Leftrightarrow\quad &2 \sum_{\neuronj \in \neurons^0} \weight_{\neuroni \neuronj}^1 \varx_{\neuronj}^0 < \refer_\neuroni^1 + \frac{1}{q},\quad &(2 \sum_{\neuronj \in \neurons^0} \weight_{\neuroni \neuronj}^1 \varx_\neuronj^0, \refer_\neuroni^1 + \frac{1}{q} \in \frac{2}{q} \zz)\\
        \Leftrightarrow\quad &2 \sum_{\neuronj \in \neurons^0} \weight_{\neuroni \neuronj}^1 \varx_{\neuronj}^0 < \sum_{\neuronj \in \neurons^0} \weight_{\neuroni \neuronj}^1 - \bias_\neuroni^1,\quad &(\text{$\refer_\neuroni^1 + \frac{1}{q}$ is the smallest element in $\frac{2}{q} \zz$}\\
        &&\text{not smaller than $\sum_{\neuronj \in \neurons^0} \weight_{\neuroni \neuronj}^1 - \bias_\neuroni^1$})\\
        \Leftrightarrow\quad &\funca_\neuroni^1(\varxb^0) < 0,&\\
        \Leftrightarrow\quad &\one_{\rr_+}(\funca_\neuroni^1(\varxb^0)) = 0 = \frac{1-\constc_\neuroni }{2}.&
    \end{aligned}
    \end{equation*}
    For $\layer \in \{2, \ldots, \layercount\}$, the same arguments hold because $2 \sum_{\neuronj \in \neurons^{\layer - 1}} \weight_{\neuroni \neuronj}^\layer \varx_\neuronj^{\layer - 1}$ and $\refer_\neuroni^\layer + 1$ are even integers, and $\refer_\neuroni^\layer + 1$ is the smallest even integer not smaller than $\sum_{\neuronj \in \neurons^{\layer - 1}} \weight_{\neuroni \neuronj}^\layer - \bias_\neuroni^\layer$. \Halmos
\end{proof}

\proof{Proof of Theorem \ref{theo:varFix}.}{
  Let   $\varxb^\layer \in \setx^{\layer}$ and let $\varxb^{\layer-1} \in \setx^{\layer-1}$ be such that $\varxb^{\layer}=\one_{\rr_+}(\funca^\layer(\varxb^{\layer - 1}))$. Then  $\varxb^{\layer-1} \in \setxo^{\layer-1}$ because $\setxo^{\layer - 1} \supseteq \setx^{\layer - 1}$. Then $\optobj \le \constc_\neuroni \refer_\neuroni/2$ implies
\[   \constc_\neuroni \sum_{\neuronj \in \neurons^{\layer - 1}} \weight_{\neuroni \neuronj}^\layer \varx_\neuronj^{\layer - 1} \le \constc_\neuroni \refer_\neuroni/2 \]
which by Lemma \ref{lem:propEqui} implies
\[ \varx_i^\ell =  \one_{\rr_+}(\funca_\neuroni^\layer(\varxb^{\layer - 1})) = \frac{1-\constc_\neuroni}{2}. \]
 Therefore, \eqref{eq:varFix} is a valid equality for $\setx^\layer$. \Halmos
}

Although checking whether a variable fixing is valid requires solving the IP \eqref{eq:varFixForm}, this IP is much easier to solve than the original verification IP as it considers only a single layer at a time and the constraints defining the outer approximation of $X^{\layer - 1}$ do not involve any ``big-$M$'' constraint \eqref{eq:varFixForm}. Furthermore, it can be solved very efficiently for certain structures of  $\setxo^{\layer-1}$ as discussed in the following paragraphs.

First, consider the case $\layer \in \{2, \ldots, \layercount\}$ and assume  $\setxo^{\layer - 1}$ is defined only by  variable fixings, e.g., as is the case if it is defined recursively only the result of Theorem \ref{theo:varFix}.
Then, \eqref{eq:varFixForm} can be solved by setting $\varx_\neuronj^{\layer - 1}$ to  $1$ if $\constc_\neuroni \weight_{\neuroni \neuronj}^\layer = 1$ and the variable $\varx_{\neuronj}^{\layer-1}$ is not fixed in $\setxo^{\layer - 1}$, and setting all other unfixed variables to $0$.

For $\layer = 1$, the special form of the set $X^0$ also allows \eqref{eq:varFixForm} to be solved efficiently. We describe here how this is done for the case $\normp = 1$. The methods for solving \eqref{eq:varFixForm} for the cases $\normp=2$ and $\normp=\infty$ are described in Appendix \ref{eq:inf2details}. For $\normp=1$, the following inequality holds because $\varxb^0 \in [0, 1]^{\neuroncount^0}$:
\begin{equation}
\label{eq:ubcx}
     \constc_\neuroni \sum_{\neuronj \in \neurons^0} \weight_{\neuroni \neuronj}^1 \varx_\neuronj^0 \le  \sum_{\neuronj \in \neurons^0} \max(\constc_\neuroni \weight_{\neuroni \neuronj}^1, 0).
\end{equation}
Inequality \eqref{eq:ubcx} is satisfied at equality by $\varxbh$ defined by
$\varxh_\neuronj =\varxbar_\neuronj$ for $j$ with $\weight_{\neuroni \neuronj}^1 = 0$,  $\varxh_\neuronj=1$ for $j$ with $\weight_{\neuroni \neuronj}^1 = \constc_\neuroni$, and $\varxh_\neuronj=0$ for  $j$ with $\weight_{\neuroni \neuronj}^1 = \ -\constc_\neuroni$.
If $\|\varxbh - \varxbbar\|_1 \le \pert$, $\varxbh \in \setx^0$, so the optimal objective value is $ \sum_{\neuronj \in \neurons^0} \max(\constc_\neuroni \weight_{\neuroni \neuronj}^1, 0)$. Otherwise, by iteratively decreasing $\hat{x}_j$ by $\frac{1}{q}$ for $j$ with $W_{ij}^1=c_i$ and $\hat{x}_j > \bar{x}_j$ or  increasing $\hat{x}_j$ by $\frac{1}{q}$ for $j$ with $W_{ij}^1=-c_i$ and $\hat{x}_j < \bar{x}_j$ we can obtain a solution $\varxb^* \in \setx^0$ satisfying $\|\varxb^* - \varxbbar\|_1 = \frac{1}{\quant}\lfloor \quant \pert\rfloor$,  $\varx^*_\neuronj=\varxbar_{\neuronj}$ for $\neuronj$ with $\weight_{\neuroni \neuronj}^1 = 0$, 
$\varx^*_\neuronj>\varxbar_{\neuronj}$ for $\neuronj$ with $\weight_{\neuroni \neuronj}^1 = c_i$, and
$\varx^*_\neuronj< \varxbar_{\neuronj}$ for $\neuronj$ with $\weight_{\neuroni \neuronj}^1 = -c_i$.
The following inequalities are satisfied by every $\varxb \in \setx^0$ and hold at equality for $\varxb^*$:
\begin{equation}
\begin{aligned}
     \constc_\neuroni \sum_{\neuronj \in \neurons^0} \weight_{\neuroni \neuronj}^1 \varx_\neuronj^0 &= \sum_{\neuronj \in \neurons^0} \constc_\neuroni \weight_{\neuroni \neuronj}^1 (\varx_\neuronj^0 - \varxbar_\neuronj) + \constc_\neuroni \sum_{\neuronj \in \neurons^0} \weight_{\neuroni \neuronj}^1 \varxbar_\neuronj&\\
    &\le \sum_{\neuronj \in \neurons^0} |\varx_\neuronj^0 - \varxbar_\neuronj| + \constc_\neuroni \sum_{\neuronj \in \neurons^0} \weight_{\neuroni \neuronj}^1 \varxbar_\neuronj &(\constc_\neuroni \weight_{\neuroni \neuronj}^1 \in \{-1, 0, 1\} )\\
    &\le \frac{1}{\quant}\lfloor \quant \pert\rfloor + \constc_\neuroni \sum_{\neuronj \in \neurons^0} \weight_{\neuroni \neuronj}^1 \varxbar_\neuronj&
\end{aligned}
\end{equation}
where the last inequality follows because $\sum_{\neuronj \in \neurons^0} |\varx_\neuronj^0 - \varxbar_\neuronj| = \|\varxb^0 - \varxbbar\|_1 \le \pert$ and $\sum_{\neuronj \in \neurons^0} |\varx_\neuronj^0 - \varxbar_\neuronj|$ is a multiple of $\frac{1}{\quant}$. It follows that  the optimal objective value of \eqref{eq:varFixForm} in this case is $\frac{1}{\quant}\lfloor \quant \pert\rfloor + \constc_\neuroni \sum_{\neuronj \in \neurons^0} \weight_{\neuroni \neuronj}^1 \varxbar_\neuronj$ and this value is achieved by $\varxb^*$.

\subsection{Two-variable Inequalities}\label{subsec:3_2}

We next seek to derive sufficient conditions for inequalities of the following form to be valid inequalities for $\setx^\layer$ for $\layer \in [\layercount]$:
\begin{equation}
    \constc_\neuroni \varx_\neuroni^\layer + \constc_\neuronk \varx_\neuronk^\layer \le \frac{\constc_\neuroni + \constc_\neuronk}{2}, \label{eq:twoVarIneq}
\end{equation}
for $\neuroni, \neuronk \in \neurons^\layer$ satisfying $\neuroni > \neuronk$, and $\constc_\neuroni, \constc_\neuronk \in \{-1, 1\}$.
The following theorem presents an IP that yields a sufficient condition for \eqref{eq:twoVarIneq} to be a valid inequality for $\setx^\layer$.

\begin{theorem}\label{theo:twoVarIneq}
    Consider $\layer \in [\layercount]$, $\neuroni, \neuronk \in \neurons^\layer$ satisfying $\neuroni < \neuronk$, and $\constc_\neuroni, \constc_\neuronk \in \{-1, 1\}$. Let $\setxo^{\layer - 1} \subset \{0, 1\}^{\neuroncount^{\layer - 1}}$ satisfy $\setxo^{\layer - 1} \supseteq \setx^{\layer - 1}$ and define
    \begin{equation}
        \optobj = \max\Big\{ \constc_\neuroni \sum_{\neuronj \in \neurons^{\layer - 1}} \weight_{\neuroni \neuronj}^\layer \varx_\neuronj^{\layer - 1}: \varxb^{\layer - 1} \in \setxo^{\layer - 1},  \constc_\neuronk \sum_{\neuronj \in \neurons^{\layer - 1}} \weight_{\neuronk \neuronj}^\layer \varx_\neuronj^{\layer - 1} \ge \constc_\neuronk \refer_\neuronk^\layer/2 \Big\} .\label{eq:twoVarIneqForm}
    \end{equation}
   If $\optobj \le \constc_\neuroni \refer_\neuroni^\layer/2$, where $\refer_\neuroni^\layer$ is defined in \eqref{eq:const}, then \eqref{eq:twoVarIneq} is a valid inequality for $\setx^\layer$.
\end{theorem}

\begin{proof}{Proof.}  
  Let   $\varxb^\layer \in \setx^{\layer}$ and let $\varxb^{\layer-1} \in \setx^{\layer-1}$ be such that $\varxb^{\layer}=\one_{\rr_+}(\funca^\layer(\varxb^{\layer - 1}))$ and observe $\varxb^{\layer-1} \in \setxo^{\layer-1}$ since $\setxo^{\layer-1} \supseteq \setx^{\layer-1}$.

  Observe that for a binary variable $\varx$ and $c \in \{-1,1\}$, it holds that $c\varx \in \{ \frac{c+1}{2},\frac{c-1}{2} \}$. 
  
  Suppose that $c_k \varx^\layer_k = \frac{c_k-1}{2}$. Then,
  \[ c_i \varx^\layer_\neuroni + c_k \varx^\layer_\neuronk \leq \frac{c_i+1}{2} + \frac{c_k-1}{2} = \frac{c_i + c_k}{2} \]
  and hence \eqref{eq:twoVarIneq} is satisfied. 

  Thus, assume now that $c_k x_k^\ell = \frac{c_k+1}{2}$, and hence $x_k^\ell = \frac{c_k + 1}{2}$ because $c_k \in \{-1,1\}$.  Then  Lemma \ref{lem:propEqui} implies 
\[ \constc_\neuronk \sum_{\neuronj \in \neurons^{\layer - 1}} \weight_{\neuronk \neuronj}^\layer \varx_\neuronj^{\layer - 1} > \constc_\neuronk \refer_\neuronk^\layer/2 \]
and hence $\varxb^{\layer-1}$ is feasible to \eqref{eq:twoVarIneqForm}. Hence,
\[ \constc_\neuroni \sum_{\neuronj \in \neurons^{\layer - 1}} \weight_{\neuroni \neuronj}^\layer \varx_\neuronj^{\layer - 1}
\leq \optobj \le \constc_\neuroni \refer_\neuroni/2 \]
which again by Lemma \ref{lem:propEqui} implies
\[ \varx^{\layer}_\neuroni = 
 \one_{\rr_+}(\funca_\neuroni^\layer(\varxb^{\layer - 1})) = \frac{1-\constc_\neuroni}{2}. \]
Thus, $\constc_{\neuroni}\varx^{\layer}_{\neuroni} = \frac{\constc_\neuroni-1}{2}$ and hence
\[ c_i \varx^\layer_\neuroni + c_k \varx^\layer_\neuronk = \frac{c_i-1}{2} + \frac{c_k+1}{2} = \frac{c_i + c_k}{2} \]
and hence gain \eqref{eq:twoVarIneq} is satisfied. \Halmos
\end{proof}

Once again, solving \eqref{eq:twoVarIneq} is expected to be much simpler than solving the original verification problem as it only includes the constraints defining the outer approximation of $\setx^{\layer - 1}$ and one other constraint and does not contain any ``big-$M$'' constraints.

\subsection{Algorithms}\label{subsec:3_3}

We now describe how we propose to use the layerwise derived valid inequalities to solve the IP problem for BNN verification.

{
\SingleSpacedXI
\begin{algorithm}
    \caption{$\verifybnn(\varxbbar, \pert)$}\label{alg:veriBnn}
    \KwIn{input vector $\varxbbar$, input perturbation $\pert$}
    \KwOut{if $\varxbbar$ is $\pert$-verified, returns TRUE, else returns FALSE}
    Initialize $\setxi = (\setxi^0, \cdots, \setxi^\layercount) \gets (\emptyset, \cdots, \emptyset)$\\
    $\setxi \gets \updateinapprox(0, \setxi, \{\varxbbar\})$\\
    Initialize $\setxo^0 \gets \setx^0$ \label{p1start}\\
    \For{$\layer = 1, \cdots, \layercount$} {
        $\setxi, \neuronsf^\layer, \setxo^\layer \gets \phaseonefixvar(\layer, \setxi, \setxo^{\layer - 1})$ 
    } \label{p1end}
    Add constraints $\varxb^\layer \in \setxo^\layer$ for $\layer \in [\layercount]$ to (\ref{eq:incorpForm}) and start the solution process but with a limit of one node, and let $\optobj$ be the best objective value (lower bound) and $\optobjbound$  be the best upper bound on the optimal objective value \label{alg:veriBnnFixVar}\\
    \uIf{$\optobjbound \leq 0$} {
        \Return TRUE
    }
     \uElseIf{$\optobj > 0$} {
        \Return FALSE
    }
    \uElse
    {
        $\setxi, \setxo^1 \gets \gentwovarineq(1, \setxi, \setxo^0, \neuronsf^1, \setxo^1)$\\
        \For{$\layer = 2, \cdots, \layercount$} {
            $\setxi, \neuronsf^\layer, \setxo^\layer \gets \phasetwofixvar(\layer, \setxi, \setxo^{\layer - 1}, \neuronsf^\layer, \setxo^\layer)$\\
            $\setxi, \setxo^\layer \gets \gentwovarineq(\layer, \setxi, \setxo^{\layer - 1}, \neuronsf^\layer, \setxo^\layer)$
        }
        Solve (\ref{eq:incorpForm}) by adding constraints $\varxb^\layer \in \setxo^\layer$ for $\layer \in [\layercount]$, and obtain $\optobj_\pert(\varxbbar)$ from the optimal objective value
    }
    \eIf{$\optobj_\pert(\varxbbar) \le 0$} {
        \Return TRUE
    }{
        \Return FALSE
    }
\end{algorithm}
}

Our proposed method is described in Algorithm \ref{alg:veriBnn}. The algorithm consists of two phases. In the first phase, we only try to fix variables, since in this case the validity verification problem \eqref{eq:varFixForm} is easy to solve. This is described in lines \ref{p1start} - \ref{p1end}, where variable fixings are identified by calling Algorithm \ref{alg:phaseOneFixVar} for each $\layer \in [\layercount]$, which outputs an outer approximation for $\setx^\layer$ ($\setxo^\layer$) based on fixing variables in the set $\neuronsf^\layer \subseteq \neurons^\layer$. Then, the generated variable fixings are added to  the IP formulation \eqref{eq:incorpForm} as constraints and \eqref{eq:incorpForm} is solved by an IP solver with a limit of one node (i.e., the IP solver only processes the root node of the branch-and-bound tree). If the IP solver solves the verification problem within that limit (e.g., by finding a feasible solution with positive objective value or proving a non-positive upper bound), then it returns the result accordingly. Otherwise, Algorithm \ref{alg:veriBnn} moves to the second phase. In this phase, two-variable inequalities are generated by Algorithm \ref{alg:genTwoVarIneq} for $\layer \in [\layercount]$ 
 and additional variable fixings are checked, which may be possible due to the improved outer approximations defined by the two-variable inequalities. For layers $\layer \in \{2,\ldots,L\}$, Algorithm \ref{alg:genTwoVarIneq} outputs an outer approximation $\setxo^\layer$ that is potentially a subset of the one generated in the first phase. The inequalities defining these outer approximations are added to the IP formulation \eqref{eq:incorpForm} and it is solved.
In order to answer the verification question, it is only necessary to determine the sign of the optimal value of \eqref{eq:incorpForm}. Thus,  when solving \eqref{eq:incorpForm} the process can be terminated if either the best objective value of a feasible solution found is positive or the best upper bound for the optimal objective value is non-positive.

{
\SingleSpacedXI
\begin{algorithm}
    \caption{$\phaseonefixvar(\layer, \setxi, \setxo^{\layer - 1})$}\label{alg:phaseOneFixVar}
    \KwIn{layer $\layer$, inner approximation $\setxi$, outer approximation $\setxo^{\layer - 1}$ for $\setx^{\layer - 1}$}
    \KwOut{$\setxi$, set $\neuronsf^\layer$ of $\neuroni \in \neurons^\layer$ where $\varx_\neuroni^\layer$ is fixed in $\setx^\layer$, outer approximation $\setxo^\layer$ for $\setx^\layer$}
    Initialize $\neuronsf^\layer \gets \emptyset$\\
    Initialize $\setxo^\layer \gets \{0, 1\}^{\neuroncount^\layer}$\\
    \For{$\neuroni \in \neurons^\layer$} {
        \For{$\constc_\neuroni \in \{ -1, 1\}$} {
            Solve the IP subproblem (\ref{eq:varFixForm}), and obtain the optimal objective value $\optobj$\\
            \If{$\layer = 1$} {
                Let $x^{i,c_i}$ be optimal solution obtained when solving (\ref{eq:varFixForm})\\
                $\setxi \gets \updateinapprox(0, \setxi, \{ x^{i,c_i}\})$
            }
            \If{$\optobj \le \constc_\neuroni \refer_\neuroni^\layer/2$} {
                $\neuronsf^\layer \gets \neuronsf^\layer \cup \Big\{\Big(\neuroni, \frac{-\constc_\neuroni + 1}{2}\Big)\Big\}$\\
                $\setxo^\layer \gets \setxo^\layer \cap \Big\{\varxb^\layer \in \{0, 1\}^{\neuroncount^\layer}: \varx_\neuroni^\layer = \frac{-\constc_\neuroni + 1}{2}\Big\}$
            }
        }
    }
    \Return $\setxi, \neuronsf^\layer, \setxo^\layer$
\end{algorithm}
}

{
\SingleSpacedXI
\begin{algorithm}
    \caption{$\phasetwofixvar(\layer, \setxi, \setxo^{\layer - 1}, \neuronsf^\layer, \setxo^\layer)$}\label{alg:phaseTwoFixVar}
    \KwIn{layer $\layer$, inner approximation $\setxi$, outer approximation $\setxo^{\layer - 1}$ for $\setx^{\layer - 1}$, set $\neuronsf^\layer$ of $\neuroni \in \neurons^\layer$ where $\varx_\neuroni^\layer$ is fixed in $\setx^\layer$, outer approximation $\setxo^\layer$ for $\setx^\layer$}
    \KwOut{$\setxi$, $\neuronsf^\layer$, $\setxo^\layer$}
    \For{$\neuroni \in \neurons^\layer \setminus \neuronsf^\layer$} {
        \For{$\constc_\neuroni \in \{ -1, 1\}$} {
            \If{$\max_{\varxb^\layer \in \setxi^\layer} \constc_\neuroni \varx_\neuroni^\layer > \frac{\constc_\neuroni - 1}{2}$} {
                \Continue
            }
            Solve the IP subproblem (\ref{eq:varFixForm}), and obtain a set $\setxf^{\layer - 1}$ of feasible solutions and the optimal objective value $\optobj$\\
            $\setxi \gets \updateinapprox(\layer, \setxi, \one_{\rr_+}(\funca^\layer(\setxf^{\layer - 1})))$\\
            \If{$\optobj \le \constc_\neuroni \refer_\neuroni^\layer/2$} {
                $\neuronsf^\layer \gets \neuronsf^\layer \cup \{\Big(\neuroni, \frac{-\constc_\neuroni + 1}{2}\Big)\}$\\
                $\setxo^\layer \gets \setxo^\layer \cap \Big\{\varxb^\layer \in \{0, 1\}^{\neuroncount^\layer}: \varx_\neuroni^\layer = \frac{-\constc_\neuroni + 1}{2}\Big\}$
            }
        }
    }
    \Return $\setxi, \neuronsf^\layer, \setxo^\layer$
\end{algorithm}
}

{
\SingleSpacedXI
\begin{algorithm}
    \caption{$\gentwovarineq(\layer, \setxi, \setxo^{\layer - 1}, \neuronsf^\layer, \setxo^\layer)$}\label{alg:genTwoVarIneq}
    \KwIn{layer $\layer$, inner approximation $\setxi$, outer approximation $\setxo^{\layer - 1}$ for $\setx^{\layer - 1}$, set $\neuronsf^\layer$ of $\neuroni \in \neurons^\layer$ where $\varx_\neuroni^\layer$ is fixed in $\setx^\layer$, outer approximation $\setxo^\layer$ for $\setx^\layer$}
    \KwOut{$\setxi$, $\setxo^\layer$}
    Initialize $\neuronsp^\layer \gets \{(\neuroni, \constc_\neuroni, \neuronk, \constc_\neuronk): \neuroni, \neuronk \in \neurons^\layer \setminus \neuronsf^\layer, \constc_\neuroni, \constc_\neuronk \in \{-1, 1\}, \neuroni < \neuronk\}$, $\failcount \gets 0$\\
    Sort $\neuronsp^\layer$ in a descending order by the score based on $\setxi^\layer$\\
    \For{$(\neuroni, \constc_\neuroni, \neuronk, \constc_\neuronk) \in \neuronsp^\layer$ in descending order of score$(\neuroni, \constc_\neuroni, \neuronk, \constc_\neuronk)$} {
        \If{$\max_{\varxb^\layer \in \setxi^\layer} (\constc_\neuroni \varx_\neuroni^\layer + \constc_\neuronk \varx_\neuronk^\layer) > \frac{\constc_\neuroni + \constc_\neuronk}{2}$} {
            \Continue
        }
        Solve the IP subproblem (\ref{eq:twoVarIneqForm}), and obtain a set $\setxf^{\layer - 1}$ of feasible solutions and the optimal objective value $\optobj$\\
        $\setxi \gets \updateinapprox(\layer, \setxi, \one_{\rr_+}(\funca^\layer(\setxf^{\layer - 1})))$\\
        \uIf{$\optobj \le \constc_\neuroni \refer_\neuroni^\layer/2$} {
            $\setxo^\layer \gets \setxo^\layer \cap \Big\{\varxb^\layer \in \{0, 1\}^{\neuroncount^\layer}: \constc_\neuroni \varx_\neuroni^\layer + \constc_\neuronk \varx_\neuronk^\layer \le \frac{\constc_\neuroni + \constc_\neuronk}{2}\Big\}$\\
            $\failcount \gets 0$
        }
        \uElse {
            $\failcount \gets \failcount + 1$
        }
        \uIf{$\failcount \geq \maxfail$}{
            \Break
        }
    }
    \Return $\setxi, \setxo^\layer$
\end{algorithm}
}
Algorithms \ref{alg:phaseOneFixVar}, \ref{alg:phaseTwoFixVar}, and \ref{alg:genTwoVarIneq} generate variable fixings in phase 1, variable fixings in phase 2, and two-variable inequalities, respectively. These algorithms check whether each candidate layerwise derived valid inequality is valid for $\setx^{\layer}$ by solving the IP subproblem \eqref{eq:varFixForm} or \eqref{eq:twoVarIneq}. When solving these IP subproblems,  if the upper bound on the optimal objective value falls below $\constc_\neuroni \refer_\neuroni^\layer/2$ the optimization process is terminated since we then know the candidate inequality is valid. On the other hand, if the lower bound (incumbent objective value) exceeds $\constc_\neuroni \refer_\neuroni^\layer/2$, the optimization process is terminated since we can then conclude that we are not able to verify validity of the candidate inequality. In Algorithms \ref{alg:phaseOneFixVar} and \ref{alg:phaseTwoFixVar}, the neuron fixed by the candidate is added to $\neuronsf^\layer$ because this set is exploited in Algorithms \ref{alg:phaseTwoFixVar} and \ref{alg:genTwoVarIneq} to retrieve candidates for layerwise derived valid inequalities.

{
\SingleSpacedXI
\begin{algorithm}
    \caption{$\updateinapprox(\layer, \setxi, \setxf^\layer)$}\label{alg:updateInApprox}
    \KwIn{layer $\layer$, inner approximation $\setxi$, set $\setxf^\layer \subseteq \setxo^{\layer}$}
    \KwOut{$\setxi$}
    $\setxi^\layer \gets \setxi^\layer \cup \setxf^\layer$\\
    \For{$\layer' = \layer + 1, \cdots, \layercount$} {
        Initialize $\setxf^{\layer'} \gets \one_{\rr_+}(\funca^{\layer'}(\setxf^{\layer' - 1}))$\\
        $\setxi^{\layer'} \gets \setxi^{\layer'} \cup \setxf^{\layer'}$
    }
    \Return $\setxi$
\end{algorithm}
}

To limit the time spent checking validity of candidates, Algorithms \ref{alg:phaseTwoFixVar} and \ref{alg:genTwoVarIneq} use an inner approximation $\setxi^\layer$ for $\one_{\rr_+}(\funca^\layer(\setxo^{\layer - 1}))$ to detect candidates that will not lead to a valid inequality, and hence avoid solving \eqref{eq:varFixForm} or \eqref{eq:twoVarIneq} for those candidates. For each candidate layerwise derived inequality, if the maximum of the left-hand side over $\setxi^\layer$ is larger than the right-hand side, this candidate is not valid for $\setx^{\layer}$, so solving the IP subproblem can be avoided. In Algorithm \ref{alg:veriBnn}, $\setxi = (\setxi^0, \cdots, \setxi^\layercount)$ is initialized by propagating the input feature vector $\varxbbar$ through the BNN. In Algorithm \ref{alg:phaseOneFixVar} these sets are expanded by propagating the optimal solution obtained when solving \eqref{eq:varFixForm} at layer 1 (which are in $\setx^0$ by definition) through the BNN, thus yielding solutions in $\setx^\layer$ which are added to $\setxi^\layer$ for $\layer=0,\ldots,L$. In Algorithms \ref{alg:phaseTwoFixVar} and \ref{alg:genTwoVarIneq}, for each layer $\layer \in \{1,\ldots, L\}$, each time we solve \eqref{eq:varFixForm} or \eqref{eq:twoVarIneq} to check validity of a candidate inequality we collect the set of feasible solutions $\setxf^{\layer-1}$ obtained by the solver, which by definition are in the set $\setxo^{\layer-1}$. These are then propagated forwarded as $\setxf^{\layer'} = \one_{\rr_+}(\funca^{\layer'}(\setxf^{\layer' - 1}))$ for $\layer'=\layer,\ldots,L$ and these sets are added to $\setxi^{\layer'}$. Since the set $\setxo^{\ell-1}$ is completely determined when processing layer $\ell$ in Algorithm \ref{alg:phaseTwoFixVar}, this process ensures that $\setxi^{\ell} \subseteq \setxo^{\ell}$ for all $\ell$, and hence if a candidate valid inequality is violated by a vector in $\setxi^{\ell}$ we can conclude that solving \eqref{eq:varFixForm} or \eqref{eq:twoVarIneq} cannot lead to a verification that the inequality is valid.

As a final strategy to limit the computational time spent solving \eqref{eq:twoVarIneq} on candidates that do not yield valid inequalities, 
Algorithm \ref{alg:genTwoVarIneq} includes a heuristic stopping rule that quits the process if it seems unlikely to yield more additional valid inequalities. The idea is to check the candidate inequalities in a sequence defined by a score that correlates with how likely they are to yield a valid inequality, and then terminate once the number of consecutive failed attempts exceeds a pre-specified limit `$\maxfail$'. The score that we use for a candidate inequality $ (i, c_i, k, c_k) \in \neuronsp^\layer$ also leverages the inner approximations we have built and is defined as:
$$
\text{score}(\neuroni, \constc_\neuroni, \neuronk, \constc_\neuronk) =    \frac{|\{\varxbh^\layer \in \setxi^\layer: \varxh_\neuroni^\layer = \frac{-\constc_\neuroni + 1}{2}\}|+ |\{\varxbh^\layer \in \setxi^\layer: \varxh_\neuronk^\layer = \frac{-\constc_\neuronk + 1}{2}\}|}{|\setxi^\layer|}.
$$
To understand the intuition for this score, consider the case of $c_i=c_k=1$, so that the candidate inequality is of the form $x^\layer_i + x^\layer_k \leq 1$, and the score sums the fraction of solutions in the inner approximation $\setxi^\ell$ which have $x^\layer_i = 0$ and which have $x^\layer_k = 0$. A high score suggests most solutions in $\setxo^\layer$ have either $x^\layer_i = 0$ or $x^\layer_k = 0$ and hence the candidate inequality might be satisfied.
The intuition behind the other cases is similar.

\section{Computational Study}\label{sec:4}

We pursue a computational study to investigate how IP methods work to solve the BNN verification problem for multiple test instances. The following five IP methods are considered in the computational study:
\begin{itemize}
    \item Many-IP: solve the IP problem \eqref{eq:indivForm} for all $\class$,
    \item 1-IP: solve the IP problem \eqref{eq:incorpForm},
    \item 1-IP+HG: solve \eqref{eq:incorpForm} by employing a constraint generation approach with \eqref{eq:singleConvHullIneq},
    \item 1-IP+Fix: solve \eqref{eq:incorpForm} by employing the variant of Algorithm \ref{alg:veriBnn} where non-root nodes are explored in solving \eqref{eq:incorpForm} in phase 1, and phase 2 is skipped,
    \item 1-IP+Fix+2Var: solve \eqref{eq:incorpForm} by employing Algorithm \ref{alg:veriBnn}.
\end{itemize}
Many-IP and 1-IP are compared to explore the impact of the technique for obtaining a linear objective. The methods 1-IP, 1-IP+HG, 1-IP+Fix, and 1-IP+Fix+2Var are compared to explore the impact of the techniques for generating layerwise derived valid inequalities.

\subsection{Test Instances}\label{subsec:4_1}

{
\SingleSpacedXI
\begin{table}
    \centering
    \begin{tabular}{c|c c c c c c}
        \hline
        \textbf{Network} & \multicolumn{1}{c}{$\mathbf{\layercount}$} & \multicolumn{1}{c}{$\mathbf{\neuroncount^1}$} & \multicolumn{1}{c}{$\mathbf{\neuroncount^2}$} & \multicolumn{1}{c}{$\mathbf{\neuroncount^3}$} & \multicolumn{1}{c}{$\mathbf{\neuroncount^4}$} & \textbf{Error Rate}\\
        \hline
        \textbf{1} & 2 & 100 & 100 & * & * & 7.53\%\\
        \textbf{2} & 2 & 200 & 100 & * & * & 5.52\%\\
        \textbf{3} & 2 & 300 & 200 & * & * & 4.19\%\\
        \textbf{4} & 3 & 100 & 100 & 100 & * & 7.33\%\\
        \textbf{5} & 3 & 200 & 100 & 100 & * & 5.08\%\\
        \textbf{6} & 3 & 300 & 200 & 100 & * & 3.68\%\\
        \textbf{7} & 4 & 200 & 100 & 100 & 100 & 5.04\%\\
        \textbf{8} & 4 & 300 & 200 & 100 & 100 & 3.46\%\\
        \textbf{9} & 4 & 500 & 300 & 200 & 100 & 2.32\%\\
        \hline
    \end{tabular}
    
    \caption{Networks in computational study}\label{tab:net}
\end{table}
}

Unless stated otherwise, the maximum perturbation from $\varxbbar$ is defined using an $\ell_1$ norm in our test instances.

Each test instance in the computational study consists of a BNN, $\varxbbar$, and $\pert$. Nine BNNs are trained for the computational study by using \cite{hubara2016binarized}'s method based on a gradient descent method with the MNIST train dataset. The MNIST dataset consists of feature vectors representing handwritten digits from 0 to 9, so $\neuroncount^{\layercount + 1} = 10$. These feature vectors are originally 784-dimensional non-negative integer vectors whose coordinates are at most 255, but they are scaled to 784-dimensional non-negative real vectors whose coordinates are at most 1, so $\neuroncount^0 = 784$ and $\quant = 255$. The number of hidden layers, the number of neurons in hidden layers, and the error rate of each network are reported in Table \ref{tab:net}. The error rate of each network is computed as the portion of feature vectors misclassified by this network in the MNIST test dataset.

For each network, we consider 10 instances, defined by 10 different feature vectors  $\varxbbar$. For each digit from 0 to 9, one feature vector in the MNIST test dataset whose $\classbar$ is this digit is randomly chosen.

{
\SingleSpacedXI
\begin{algorithm}
    \caption{Perturbation Selection}\label{alg:pertSel}
    \KwIn{$\pertinit$, $\maxiter$}
    \KwOut{$\pertl$,$\pertu$: lower and upper bound on maximum value of $\pert$ for which $\bar{x}$ is $\pert$-verified}
    $\pertl \gets 0, \pertu \gets NULL$\\
    $\pert \gets \pertinit$\\
    \For{$i = 1, \ldots, \maxiter$}{
    Attempt to solve \eqref{eq:bnnVeri} with a time limit for $\varxbar$ and $\epsilon$ and let $\bar{z}$ be the best upper bound obtained. \\
    \eIf{$\bar{z} \leq 0$}{
        $\pertl \gets \pert$\\
        \eIf{$\pertu$ is NULL}{
            $\pert \gets 2 \pert$
        }{
            $\pert \gets \frac{\pert + \pertu}{2}$
        }
    }{
        $\pertu \gets \pert$\\
        $\pert \gets \frac{\pert + \pertl}{2}$
    }
    }
    \Return $\pert$, $\pertl$, $\pertu$
\end{algorithm}
}
For each network and feature vector $\varxbbar$, six values are obtained for $\pert$ by employing Algorithm \ref{alg:pertSel} with $\maxiter = 6$. Algorithm \ref{alg:pertSel} is a binary search that attempts to find the largest $\pert$ under which $\varxbbar$ can be $\pert$-verified using a given method and time limit in each iteration. In Algorithm \ref{alg:pertSel}, if $\varxbbar$ is $\pert$-verified for the current $\pert$, $\pert$ is increased based on the smallest input perturbation $\pertu$ under which $\varxbbar$ may not be $\pertu$-verified. Otherwise, $\pert$ is decreased to the average of $\pert$ and the largest input perturbation $\pertl$ under which $\varxbbar$ is $\pertl$-verified. Different IP methods for the BNN verification may result in different sequences of $\pert$ in Algorithm \ref{alg:pertSel} based on their ability to solve the verification problem at a given $\pert$ within the time limit. This approach is used for determining the $\pert$ values in the test instances in order to find challenging instances, i.e., those in which determining whether $\varxbbar$ is $\pert$-verified is nontrivial. The initial value $\pertinit$ is set to $1$ in the case that the maximum perturbation from $\varxbbar$ is defined by an $\ell_1$ norm, $\frac{1}{255}$ in the case that the maximum perturbation is defined by $\ell_\infty$ norm, and $\frac{1}{32}$ in the case that the maximum perturbation is defined by $\ell_2$ norm. 

\subsection{Implementation Details}\label{subsec:4_2}

In the IP method 1-IP+HG, \eqref{eq:singleConvHullIneq} are generated by solving the LP relaxation problem of the IP problem \eqref{eq:incorpForm} and adding violated inequalities \eqref{eq:singleConvHullLowerBoundIneq} and \eqref{eq:singleConvHullUpperBoundIneq} in an iterative manner. In each iteration, for each $\layer \in \{2, \ldots, \layercount\}$ and $\neuroni \in \neurons^\layer$, an inequality \eqref{eq:singleConvHullUpperBoundIneq} with the largest violation is added to \eqref{eq:incorpForm}. Likewise, \eqref{eq:singleConvHullLowerBoundIneq} with the largest violation is added to \eqref{eq:incorpForm} as a constraint. This iteration is repeated until either no violated inequalities are found or the optimal objective value of the LP relaxation problem does not improve by 1\% over the last 10 iterations. Our test instances are different than those used in \cite{han2021single} (ours have more layers in the BNN and non-binary inputs) and we obtain qualitatively different results than those presented in \cite{han2021single}. Thus, to validate our implementation, in Appendix \ref{app:singneuron} we present results of additional experiments with the method 1-IP+HG on the instances used in \cite{han2021single} which indicates that our implementation achieves qualitatively similar results to what is reported in \cite{han2021single} on those instances.

In the IP method 1-IP+Fix and 1-IP+Fix+2Var, in Algorithm \ref{alg:phaseOneFixVar}, the IP subproblem \eqref{eq:varFixForm} is solved using the method described at the end of  Section \ref{subsec:3_1} (i.e., without an IP solver). 

In 1-IP+Fix+2Var, two-variable inequalities for the first hidden layer are not generated, and variable fixings for the second hidden layer are not generated in phase 2 because two-variable inequalities for the first hidden layer are unlikely to be generated for the test instances. The networks are dense, and $\pert$ is much smaller than $\neuroncount^0 = 784$ in test instances, so for $\neuroni, \neuronk \in \neurons^1$ satisfying $\neuroni < \neuronk$ and $\constc_\neuroni, \constc_\neuronk \in \{-1, 1\}$, it is likely to find $\varxbh^0 \in \setx^0$ where $2 \constc_\neuroni \sum_{\neuronj \in \neurons^0} \weight_{\neuroni \neuronj}^1 \varxh_\neuronj^0$ and $2 \constc_\neuronk \sum_{\neuronj \in \neurons^0} \weight_{\neuronk \neuronj}^1 \varxh_\neuronj^0$ achieve their maximum over $\setx^0$. If both $\varx_\neuroni^1$ and $\varx_\neuronk^1$ are not fixed, $\varxbh^0$ violates \eqref{eq:twoVarIneq}, so two-variable inequalities for the first hidden layer are rarely generated. As a result of not generating two-variable inequalities for the first hidden layer, variable fixings for the second hidden layer cannot be generated in phase 2. We use $\maxfail = 100$ as the number of consecutive failures after which we terminate attempting to generate two-variable inequalities.

The time limit for all IP methods is 3600 seconds. In the IP method Many-IP, the time limit to solve the IP problem \eqref{eq:indivForm} for each $\class$ is set to 400 seconds because there are nine alternative classes. In 1-IP+HG, 1-IP+Fix, and 1-IP+Fix+2Var, a time limit on the phase for generating valid inequalities of 2700 seconds is imposed.  The time limit to solve \eqref{eq:incorpForm} after generating the inequalities is set to 3600 seconds \textit{subtract} the time spent generating valid inequalities. If the time limit is hit, the best upper bound for the optimal objective value is used instead of the best objective value to answer the BNN verification problem because $\varxbbar$ is $\pert$-verified if and only if the best upper bound is non-positive.

All IP methods are implemented in Python, and Gurobi 10.0.3 is used as the IP solver. All experiments in the computational study are run on an Ubuntu desktop with 32 GB RAM and 16 Intel Core i7-10700 CPUs running at 2.90 GHz.

\subsection{Results}\label{subsec:4_3}

We first collect the maximum $\pert$ under which $\varxbbar$ is $\pert$-verified for each IP method, network, and $\varxbbar$ in the case that the maximal perturbation from $\varxbbar$ is defined using an $\ell_1$ norm. The maximum $\pert$ is obtained from using Algorithm \ref{alg:pertSel} with $\maxiter = 6$. For the same test instance, one IP method may succeed in verifying the BNN, but another IP method fails because it may not be able to verify the BNN within the time limit. As a result, the set of $\pert$ defining the six instances may differ by IP methods for the same network and $\varxbbar$, and the maximum $\pert$ may differ.

{
\SingleSpacedXI
\begin{table}
    \centering
    \begin{tabular}{c|S S S S S}
        \hline
        \multirow{3}{*}{\textbf{Network}} & \textbf{Many-IP} & \textbf{1-IP} & \textbf{1-IP} & \textbf{1-IP} & \textbf{1-IP}\\
        & & & \textbf{+HG} & \textbf{+Fix} & \textbf{+Fix}\\
        & & & & & \textbf{+2Var}\\
        \hline
        \textbf{1} & 7.23 & 7.23 & 7.23 & 7.23 & 7.23\\
        \textbf{2} & 6.15 & 6.48 & 6.48 & 6.48 & 6.48\\
        \textbf{3} & 1.60 & 5.45 & 5.32 & 6.28 & 6.28\\
        \textbf{4} & 6.07 & 6.28 & 6.28 & 6.18 & 6.28\\
        \textbf{5} & 5.55 & 6.35 & 6.35 & 6.55 & 7.35\\
        \textbf{6} & 2.95 & 4.30 & 3.95 & 4.43 & 5.45\\
        \textbf{7} & 4.51 & 5.06 & 5.06 & 5.29 & 6.49\\
        \textbf{8} & 2.51 & 3.89 & 3.78 & 3.99 & 5.28\\
        \textbf{9} & 0.64 & 2.17 & 2.06 & 2.34 & 3.88\\
        \hline
    \end{tabular}
    
    \caption{Average maximum $\pert$ under which $\varxbbar$ is $\pert$-verified by each method within the time limit.}\label{tab:maxPert}
\end{table}
}
For each IP method and network, the arithmetic mean of the maximum $\pert$ over ten $\varxbbar$ is presented in Table \ref{tab:maxPert}. Many-IP achieves the smallest maximum $\pert$, and 1-IP+HG achieves a smaller maximum $\pert$ than the 1-IP. Except for Network 4,  1-IP+Fix results in a larger maximum $\pert$ than 1-IP. Method 1-IP+Fix+2Var yields the largest maximum $\pert$ in every case. These results indicate that the using the proposed variable fixing approach yields modest improvement in the ability to find the maximum $\pert$ at which an input can be verified, and using two-variable layer-wise derived inequalities yields significantly more improvement.
{
\SingleSpacedXI
\begin{table}
    \centering
    \begin{tabular}{c|c|S S}
        \hline
        \multirow{2}{*}{\textbf{Network}} & \multirow{2}{*}{\begin{tabular}{c} \textbf{\# of}\\ \textbf{Instances} \end{tabular}} & \textbf{Many-IP} & \textbf{1-IP}\\
        & & &\\
        \hline
        \textbf{1} & 60 & 153.1(3) & 29.4\\
        \textbf{2} & 57 & 721.9(12) & 110.5\\
        \textbf{3} & 20 & 1420.8(10) & 358.9\\
        \textbf{4} & 59 & 269.7(8) & 58.4(2)\\
        \textbf{5} & 56 & 986.5(18) & 286.5(8)\\
        \textbf{6} & 39 & 1941.8(18) & 568.3(7)\\
        \textbf{7} & 56 & 972.3(22) & 279.7(12)\\
        \textbf{8} & 34 & 1654.4(17) & 503.1(7)\\
        \textbf{9} & 18 & 2852.2(13) & 710.8(3)\\
        \hline
    \end{tabular}
    
    \caption{Verification time (seconds) of Many-IP and 1-IP}\label{tab:manyIpOneIpVeriTime}
\end{table}
}

We next investigate in more detail the ability of the different methods to solve verification problems for various values of $\pert$.
We first compare Many-IP and 1-IP. To do so, for each network and feature vector we consider the subset of $\pert$ values that are solved by both of these methods, and
compute the shifted geometric mean of solution times to solve the IP problem \eqref{eq:incorpForm} (for 1-IP) or to solve \eqref{eq:indivForm} for all $\class$ (for Many-IP), with a shift of 1.
These results are presented in Table \ref{tab:manyIpOneIpVeriTime}, where for Many-IP, the number in parenthesis is the number of test instances where the time limit is hit in solving \eqref{eq:indivForm} for at least one $\class$ and for 1-IP, the number in parenthesis is the number of test instances where the time limit is hit in solving \eqref{eq:incorpForm}. We find that 1-IP solves the verification problem much more quickly than Many-IP, indicating that our technique for obtaining a linear objective leads to faster BNN verification. Considering the superiority of 1-IP over Many-IP, we exclude Many-IP from further comparisons.
{
\SingleSpacedXI
\begin{table}
    \centering
    \begin{tabular}{c|c|S S S S}
        \hline
        \multirow{3}{*}{\textbf{Network}} & \multirow{3}{*}{\begin{tabular}{c} \textbf{\# of}\\ \textbf{Instances} \end{tabular}} & \textbf{1-IP} & \textbf{1-IP} & \textbf{1-IP} & \textbf{1-IP}\\
        & & & \textbf{+HG} & \textbf{+Fix} & \textbf{+Fix}\\
        & & & & & \textbf{+2Var}\\
        \hline
        \textbf{1} & 60 & 29.4 & 31.6 & 5.6 & 7.0\\
        \textbf{2} & 60 & 119.6 & 124.3 & 16.8 & 16.2\\
        \textbf{3} & 52 & 779.5(10) & 842.5(13) & 58.4(4) & 62.7(4)\\
        \textbf{4} & 59 & 58.4(2) & 75.4(1) & 15.2(2) & 11.3\\
        \textbf{5} & 55 & 296.1(9) & 377.9(9) & 61.7(7) & 30.1(1)\\
        \textbf{6} & 46 & 612.2(13) & 942.1(14) & 130.4(12) & 56.3(2)\\
        \textbf{7} & 52 & 267.8(12) & 376.7(15) & 84.3(12) & 36.1(4)\\
        \textbf{8} & 36 & 582.5(12) & 722.1(12) & 119.9(11) & 36.0(2)\\
        \textbf{9} & 25 & 1183.8(11) & 1445.2(11) & 276.1(9) & 68.3(1)\\
        \hline
    \end{tabular}
    
    \caption{Verification time (seconds) of IP methods based on 1-IP}\label{tab:oneIpVeriTime}
\end{table}
}

In Table \ref{tab:oneIpVeriTime} we report the geometric average verification times over the instances tested by all IP methods based on 1-IP. The set of test instances in this case is based on the set of $\pert$ values that were solved by all 1-IP methods.
This set of test instances may differ from the set used in the Many-IP and 1-IP comparison due to the exclusion of Many-IP from this comparison.  The number in parenthesis is the number of test instances where the time limit is hit when solving \eqref{eq:incorpForm}. We find that 1-IP+HG's verification times are larger than 1-IP's, and 1-IP+Fix's are shorter than 1-IP's. Thus, we conclude that using inequalities \eqref{eq:singleConvHullIneq} does not accelerate BNN verification, but variable fixings do. We also find that 1-IP+Fix+2Var's verification times are similar to 1-IP+Fix's for Networks 1-4, and 1-IP+Fix+2Var's verification times are shorter than 1-IP+Fix's for Networks 5-9. This suggests that the two-variable layerwise derived inequalities are more helpful for BNN verification of networks with more hidden layers.
{
\SingleSpacedXI
\begin{table}
    \centering
    \begin{tabular}{c|S S S S}
        \hline
        \multirow{3}{*}{\textbf{Network}} & \multicolumn{1}{c}{\textbf{1-IP}} & \multicolumn{1}{c}{\textbf{1-IP}} & \multicolumn{1}{c}{\textbf{1-IP}} & \multicolumn{1}{c}{\textbf{1-IP}}\\
        & & \multicolumn{1}{c}{\textbf{+HG}} & \multicolumn{1}{c}{\textbf{+Fix}} & \multicolumn{1}{c}{\textbf{+Fix}}\\
        & & & & \multicolumn{1}{c}{\textbf{+2Var}}\\
        \hline
        \textbf{1} & \SI{116.4}{\percent} & \SI{116.4}{\percent} & \SI{32.5}{\percent} & \SI{18.2}{\percent}\\
        \textbf{2} & \SI{120.8}{\percent} & \SI{120.8}{\percent} & \SI{36.0}{\percent} & \SI{19.9}{\percent}\\
        \textbf{3} & \SI{131.2}{\percent} & \SI{132.9}{\percent} & \SI{50.3}{\percent} & \SI{30.5}{\percent}\\
        \textbf{4} & \SI{125.4}{\percent} & \SI{124.9}{\percent} & \SI{49.1}{\percent} & \SI{13.9}{\percent}\\
        \textbf{5} & \SI{137.4}{\percent} & \SI{137.8}{\percent} & \SI{67.4}{\percent} & \SI{20.5}{\percent}\\
        \textbf{6} & \SI{150.7}{\percent} & \SI{153.3}{\percent} & \SI{82.4}{\percent} & \SI{25.6}{\percent}\\
        \textbf{7} & \SI{144.1}{\percent} & \SI{150.2}{\percent} & \SI{82.0}{\percent} & \SI{31.7}{\percent}\\
        \textbf{8} & \SI{178.8}{\percent} & \SI{179.3}{\percent} & \SI{93.8}{\percent} & \SI{26.9}{\percent}\\
        \textbf{9} & \SI{180.2}{\percent} & \SI{180.2}{\percent} & \SI{109.2}{\percent} & \SI{13.8}{\percent}\\
        \hline
    \end{tabular}
    
    \caption{LP gap of IP methods based on 1-IP}\label{tab:oneIpLpGap}
\end{table}
}

We next investigate the LP gaps of \eqref{eq:incorpForm} over instances tested by all IP methods based on 1-IP. The LP gap is defined as $(\optobjlp - \optobj)/\objupperbound$ where $\optobjlp$ is the optimal objective value of the LP relaxation problem of \eqref{eq:incorpForm} after adding any valid inequalities introduced by the method, $\optobj$ is the best objective value from a feasible solution to \eqref{eq:incorpForm} obtained by the method, and 
\[ \objupperbound = \max_{\substack{\varxb^\layercount \in \{0, 1\}^{\neuroncount^\layercount}\\ \class \in \neurons^{\layercount + 1}}} (\funca_\class^{\layercount + 1}(\varxb^\layercount) - \funca_{\classbar}^{\layercount + 1}(\varxb^\layercount)). \] 
We use $\objupperbound$ as the denominator instead of $|\optobj|$ because $\optobj$ may be near zero for some instances, which would skew interpretation of these relative optimality gaps when averaged over different instances. The quantity $\objupperbound$ is positive by definition and is a natural upper bound on the optimal value of \eqref{eq:incorpForm}. Table \ref{tab:oneIpLpGap} reports the average LP gaps for each method. We find that the LP gaps of 1-IP+HG and 1-IP are similar, 1-IP+Fix yields significant improvement over 1-IP, and 1-IP-Fix+2Var yields further significant improvement beyond that.  Although 1-IP+HG does not improve LP gaps on these instances, we remark that in the supplemental experiments reported in Appendix \ref{app:singneuron} we found that on the instances used in \cite{han2021single} the use of the single-neuron valid inequalities did yield modest reduction in the LP gap, although even on these instances this reduction still did not translate into reduced time to solve the verification problem.

{
\SingleSpacedXI
\begin{table}
    \centering
    \begin{tabular}{l|S S S S}
        \hline
        & \multicolumn{1}{c}{\textbf{1-IP}} & \multicolumn{1}{c}{\textbf{1-IP}} & \multicolumn{1}{c}{\textbf{1-IP}} & \multicolumn{1}{c}{\textbf{1-IP}}\\
        & & \multicolumn{1}{c}{\textbf{+HG}} & \multicolumn{1}{c}{\textbf{+Fix}} & \multicolumn{1}{c}{\textbf{+Fix}}\\
        & & & & \multicolumn{1}{c}{\textbf{+2Var}}\\
        \hline
        \textbf{Preprocessing time (seconds)} & 0.0 & 10.8 & 7.0 & 15.2\\
        \textbf{Verification time (seconds)} & 211.7 & 258.9 & 40.6 & 25.5\\
        \textbf{\# of solved instances (veri.)} & 376 & 370 & 388 & 431\\
        \textbf{\# of nodes (veri.)} & 9228.2 & 9501.0 & 78.0 & 7.3\\
        \textbf{Optimization time (seconds)} & 298.2 & 352.0 & 62.6 & 36.7\\
        \textbf{\# of solved instances (opt.)} & 347 & 347 & 359 & 406\\
        \textbf{\# of nodes (opt.)} & 16389.5 & 16540.8 & 241.1 & 31.4\\
        \textbf{Optimality gap} & \SI{22.5}{\percent} & \SI{24.4}{\percent} & \SI{19.2}{\percent} & \SI{4.8}{\percent}\\
        \hline
    \end{tabular}
    
    \caption{Other metrics of IP methods based on 1-IP (445 instances)}\label{tab:oneIp}
\end{table}
}

Table \ref{tab:oneIp} presents additional results for the IP methods based on 1-IP. These results are averaged over all 445 instances solved by all these methods.
The row `Preprocessing time' presents the shifted geometric mean of times to generate layerwise derived valid inequalities (resp. \eqref{eq:singleConvHullIneq}) for 1-IP+Fix and 1-IP+Fix+2Var (resp. 1-IP+HG).
The row `verification time' is the shifted geometric mean of times to solve the verification problem -- if a method does not solve an instance within the time limit, then the time limit is used for that instance. The row `\# of solved instances (veri.)' reports the number of instances for which each IP method solved the verification problem within the time limit. Row `\# of nodes (veri.)' is the shifted geometric mean of the number of nodes in the branch-and-bound tree when solving the verification problem (again, for instances hitting the time limit, this is just the number of nodes explored within the limit). In all the verification times the solution of \eqref{eq:incorpForm} is terminated as soon as the verification question is answered (e.g., if the upper bound becomes non-positive or the lower bound becomes positive). To provide further insight into the quality of these formulations we also attempted to solve \eqref{eq:incorpForm} without terminating once the verification question is answered. The results of this experiment are reported in the rows `Optimization time', ` (\# of solved instances (opti.)', `(\# of nodes (opti.))', which are defined analogously as the verification results. Row `optimality gap' presents the arithmetic mean of relative optimality gaps obtained in solving \eqref{eq:incorpForm} without the early termination, where the relative optimality gap is defined as $(\optobjbound - \optobj)/\objupperbound$, where $\optobjbound$ is the best bound for the optimal objective value of \eqref{eq:incorpForm} obtained by the method. A shift is set to 1 in every shifted geometric mean. These results confirm the significant time reduction obtained using variable fixing and the proposed two-variable inequalities, and indicate that this reduction is obtained thanks to the dramatic reduction in the number of branch-and-bound nodes required to either answer the verification problem or to solve the optimization problem. The results also reinforce that the single-neuron inequalities used in 1-IP+HG do not yield improvement.

{
\SingleSpacedXI
\begin{table}
    \centering
    \begin{tabular}{c|S S S S S}
        \hline
        \multirow{3}{*}{\textbf{Norm}} & \textbf{Many-IP} & \textbf{1-IP} & \textbf{1-IP} & \textbf{1-IP} & \textbf{1-IP}\\
        & & & \textbf{+HG} & \textbf{+Fix} & \textbf{+Fix}\\
        & & & & & \textbf{+2Var}\\
        \hline
        $\mathbf{\ell_\infty}$ & 0.018 & 0.026 & 0.020 & 0.025 & 0.032\\
        $\mathbf{\ell_2}$ & 0.029 & 0.051 & 0.051 & 0.247 & 0.312\\
        \hline
    \end{tabular}
    
    \caption{Maximum $\pert$ under which $\varxbbar$ is $\pert$-verified for $\ell_\infty$ norm and $\ell_2$ norm}\label{tab:addMaxPert}
\end{table}
}

All results so far have been reported for the case in which the maximum perturbation from $\varxbbar$ is defined using the $\ell_1$ norm. To illustrate the versatility of the IP approach, we also conducted analogous experiments using the $\ell_\infty$ and $\ell_2$ norm. The detailed results of these experiments are presented in Appendix \ref{app:inf2norm}, and we present a summary here. Table \ref{tab:addMaxPert} presents the arithmetic mean of the maximum $\pert$ over ten feature vectors $\varxbbar$ for each IP method on Network 5. Once gain, Many-IP achieves the smallest maximum $\pert$ in both cases. Method 1-IP+HG again yields smaller maximum $\pert$ than 1-IP in the case of $\normp = \infty$ and the same in the case of $\normp = 2$. 1-IP+Fix yields similar maximum $\pert$ to 1-IP in the case of $\normp = \infty$ and larger maximum $\pert$ than 1-IP in the case of $\normp = 2$. Thus, it seems that generating variable fixings alone does not improve BNN verification by the IP methods in the case of $\normp = \infty$, but it does in the case of $\normp = 2$. 1-IP+Fix+2Var yields the largest maximum $\pert$ in both cases, showing that generating layerwise derived valid inequalities enables verifying the BNN against a higher range of $\pert$ in the case of $\normp = \infty$ and $\normp = 2$.
\section{Conclusion}\label{sec:5}

In this paper, we investigate an IP method that exploits the technique for obtaining a linear objective and a technique for generating layerwise derived valid inequalities to solve the BNN verification problem. Our computational study shows that our IP method verifies BNNs against a higher range of input perturbation than existing IP methods. The technique for obtaining a linear objective leads to solving the BNN verification problem faster by solving a single IP problem instead of multiple IP problems for each alternative class. The technique for generating layerwise derived valid inequalities enables more efficient BNN verification by yielding smaller LP gaps at the expense of solving IP subproblems involving a single layer.

One future direction on IP methods for the BNN verification problem is investigating valid inequalities that involve more than two variables, as the inequalities we investigated contain at most two variables. 
Another future direction related to the BNN verification problem is counterfactual explanations for BNNs. As discussed in Section \ref{sec:1}, these counterfactual explanations can be obtained by solving the BNN verification problem for various input perturbations. Thus, a possible direction of research for counterfactual explanations for BNNs is to consider whether information from these related instances can be shared to accelerate the overall process. For example, it may be possible to reduce the number of candidates for layerwise derived valid inequalities to check.

\bibliographystyle{informs2014}
\bibliography{reference}

\begin{thebibliography}{26}
\providecommand{\natexlab}[1]{#1}
\providecommand{\url}[1]{\texttt{#1}}
\providecommand{\urlprefix}{URL }

\bibitem[{Amir et~al.(2021)Amir, Wu, Barrett, \protect\BIBand{} Katz}]{amir2021smt}
Amir G, Wu H, Barrett C, Katz G (2021) An smt-based approach for verifying binarized neural networks. Groote JF, Larsen KG, eds., \emph{Tools and Algorithms for the Construction and Analysis of Systems}, 203--222 (Springer).

\bibitem[{Balas(1985)}]{balas1985disjunctive}
Balas E (1985) Disjunctive programming and a hierarchy of relaxations for discrete optimization problems. \emph{SIAM Journal on Algebraic Discrete Methods} 6(3):466--486.

\bibitem[{Bernardelli et~al.(2023)Bernardelli, Gualandi, Lau, \protect\BIBand{} Milanesi}]{bernardelli2023bemi}
Bernardelli AM, Gualandi S, Lau HC, Milanesi S (2023) The bemi stardust: A structured ensemble of binarized neural networks. Sellmann M, Tierney K, eds., \emph{Learning and Intelligent Optimization}, 443--458 (Springer).

\bibitem[{Contardo et~al.(2024)Contardo, Fukasawa, Rousseau, \protect\BIBand{} Vidal}]{contardo2024optimal}
Contardo C, Fukasawa R, Rousseau LM, Vidal T (2024) Optimal counterfactual explanations for k-nearest neighbors using mathematical optimization and constraint programming. Basu A, Mahjoub AR, Salazar~Gonz{\'a}lez JJ, eds., \emph{Combinatorial Optimization}, 318--331 (Springer).

\bibitem[{Fischetti \protect\BIBand{} Jo(2018)}]{fischetti2018deep}
Fischetti M, Jo J (2018) Deep neural networks and mixed integer linear optimization. \emph{Constraints} 23(3):296--309.

\bibitem[{Geiger \protect\BIBand{} Team(2020)}]{geiger2020larq}
Geiger L, Team P (2020) Larq: An open-source library for training binarized neural networks. \emph{Journal of Open Source Software} 5(45):1746.

\bibitem[{Han \protect\BIBand{} G{\'o}mez(2021)}]{han2021single}
Han S, G{\'o}mez A (2021) Single-neuron convexification for binarized neural networks. https://optimization-online.org/wp-content/uploads/2021/05/8419.pdf, accessed May 27, 2021.

\bibitem[{Hubara et~al.(2016)Hubara, Courbariaux, Soudry, El-Yaniv, \protect\BIBand{} Bengio}]{hubara2016binarized}
Hubara I, Courbariaux M, Soudry D, El-Yaniv R, Bengio Y (2016) Binarized neural networks. Lee D, Sugiyama M, von Luxburg U, Guyon I, Garnett R, eds., \emph{Adv. in Neural Information Processing Systems}, volume~29, 4107–--4115 (Curran Associates, Inc.).

\bibitem[{Ivashchenko et~al.(2023)Ivashchenko, Choi, Nguyen, \protect\BIBand{} Tran}]{ivashchenko2023verifying}
Ivashchenko M, Choi SW, Nguyen LV, Tran HD (2023) Verifying binary neural networks on continuous input space using star reachability. \emph{The IEEE International Conf. on Formal Methods in Software Engineering}, 7--17 (IEEE).

\bibitem[{Jia \protect\BIBand{} Rinard(2020)}]{jia2020efficient}
Jia K, Rinard M (2020) Efficient exact verification of binarized neural networks. Larochelle H, Ranzato M, Hadsell R, Balcan MF, Lin HT, eds., \emph{Adv. in Neural Information Processing Systems}, volume~33, 1782--1795 (Curran Associates, Inc.).

\bibitem[{Khalil et~al.(2019)Khalil, Gupta, \protect\BIBand{} Dilkina}]{khalil2018combinatorial}
Khalil EB, Gupta A, Dilkina B (2019) Combinatorial attacks on binarized neural networks. \emph{International Conf. on Learning Representations}, \urlprefix\url{https://openreview.net/forum?id=S1lTEh09FQ}.

\bibitem[{Kov{\'a}sznai et~al.(2021)Kov{\'a}sznai, Gajd{\'a}r, \protect\BIBand{} Narodytska}]{kovasznai2021portfolio}
Kov{\'a}sznai G, Gajd{\'a}r K, Narodytska N (2021) Portfolio solver for verifying binarized neural networks. \emph{Annales Mathematicae et Informaticae} 53:183--200.

\bibitem[{Kung et~al.(2018)Kung, Zhang, van~der Wal, Chai, \protect\BIBand{} Mukhopadhyay}]{kung2018efficient}
Kung J, Zhang D, van~der Wal G, Chai S, Mukhopadhyay S (2018) Efficient object detection using embedded binarized neural networks. \emph{Journal of Signal Processing Systems} 90(6):877--890.

\bibitem[{Lubczyk \protect\BIBand{} Neto(2024)}]{lubczyk2024neuron}
Lubczyk D, Neto J (2024) Neuron pairs in binarized neural networks robustness verification via integer linear programming. Basu A, Mahjoub AR, Salazar~Gonz{\'a}lez JJ, eds., \emph{Combinatorial Optimization}, 305--317 (Springer).

\bibitem[{Ma et~al.(2019)Ma, Xiong, Hu, \protect\BIBand{} Ma}]{ma2019efficient}
Ma Y, Xiong H, Hu Z, Ma L (2019) Efficient super resolution using binarized neural network. \emph{The IEEE/CVF Conf. on Computer Vision and Pattern Recognition Workshops}, 694–--703 (IEEE).

\bibitem[{Martinez et~al.(2020)Martinez, Yang, Bulat, \protect\BIBand{} Tzimiropoulos}]{martinez2020training}
Martinez B, Yang J, Bulat A, Tzimiropoulos G (2020) Training binary neural networks with real-to-binary convolutions. \emph{International Conf. on Learning Representations}.

\bibitem[{McDanel et~al.(2017)McDanel, Teerapittayanon, \protect\BIBand{} Kung}]{mcdanel2017embedded}
McDanel B, Teerapittayanon S, Kung HT (2017) Embedded binarized neural networks. Gunningberg P, Voigt T, Mottola L, Lu C, eds., \emph{Proc. of the International Conf. on Embedded Wireless Systems and Networks}, 168--173 (Junction Publishing).

\bibitem[{Mothilal et~al.(2020)Mothilal, Sharma, \protect\BIBand{} Tan}]{mothilal2020explaining}
Mothilal R, Sharma A, Tan C (2020) Explaining machine learning classifiers through diverse counterfactual explanations. Hildebrandt M, Castillo C, Celis E, Ruggieri S, Taylor L, Zanfir-Fortuna G, eds., \emph{Conf. on Fairness, Accountability, and Transparency}, 607--617 (ACM).

\bibitem[{Narodytska et~al.(2018)Narodytska, Kasiviswanathan, Ryzhyk, Sagiv, \protect\BIBand{} Walsh}]{narodytska2018verifying}
Narodytska N, Kasiviswanathan S, Ryzhyk L, Sagiv M, Walsh T (2018) Verifying properties of binarized deep neural networks. McIlraith S, Weinberger K, eds., \emph{Proc. of the AAAI Conf. on Artificial Intelligence}, volume~32, 6615–--6624 (AAAI).

\bibitem[{Shih et~al.(2019)Shih, Darwiche, \protect\BIBand{} Choi}]{shih2019verifying}
Shih A, Darwiche A, Choi A (2019) Verifying binarized neural networks by angluin-style learning. Janota M, Lynce I, eds., \emph{Theory and Applications of Satisfiability Testing}, 354--370 (Springer).

\bibitem[{Shridhar et~al.(2020)Shridhar, Jain, Agarwal, \protect\BIBand{} Kleyko}]{shridhar2020end}
Shridhar K, Jain H, Agarwal A, Kleyko D (2020) End to end binarized neural networks for text classification. Moosavi NS, Fan A, Shwartz V, Glava{\v{s}} G, Joty S, Wang A, Wolf T, eds., \emph{Proc. of SustaiNLP: Workshop on Simple and Efficient Natural Language Processing}, 29--34 (Association for Computational Linguistics).

\bibitem[{Tang et~al.(2017)Tang, Hua, \protect\BIBand{} Wang}]{tang2017train}
Tang W, Hua G, Wang L (2017) How to train a compact binary neural network with high accuracy? Singh S, Markovitch S, eds., \emph{Proc. of the AAAI Conf. on Artificial Intelligence}, volume~31, 2625–--2631 (AAAI).

\bibitem[{Toro~Icarte et~al.(2019)Toro~Icarte, Illanes, Castro, Cire, McIlraith, \protect\BIBand{} Beck}]{toro2019training}
Toro~Icarte R, Illanes L, Castro M, Cire A, McIlraith S, Beck C (2019) Training binarized neural networks using mip and cp. Schiex T, de~Givry S, eds., \emph{Principles and Practice of Constraint Programming}, 401--417 (Springer).

\bibitem[{Vivier-Ardisson et~al.(2024)Vivier-Ardisson, Forel, Parmentier, \protect\BIBand{} Vidal}]{vivier2024cf}
Vivier-Ardisson G, Forel A, Parmentier A, Vidal T (2024) Cf-opt: Counterfactual explanations for structured prediction. https://arxiv.org/pdf/2405.18293, accessed May 28, 2024.

\bibitem[{Wachter et~al.(2018)Wachter, Mittelstadt, \protect\BIBand{} Russell}]{wachter2018counterfactual}
Wachter S, Mittelstadt B, Russell C (2018) Counterfactual explanations without opening the black box: Automated decisions and the gdpr. \emph{Harvard Journal of Law and Technology} 31(2):841--888.

\bibitem[{Zhang et~al.(2021)Zhang, Zhao, Chen, Song, \protect\BIBand{} Chen}]{zhang2021bdd4bnn}
Zhang Y, Zhao Z, Chen G, Song F, Chen T (2021) Bdd4bnn: A bdd-based quantitative analysis framework for binarized neural networks. Silva A, Leino KRM, eds., \emph{Computer Aided Verification}, 175--200 (Springer).

\end{thebibliography}

\setcounter{section}{0}
\renewcommand{\thesection}{\Alph{section}}
\section{Appendix}

\subsection{Proof of Lemma \ref{lem:linPropConstr}}\label{subsec:a_1}

We prove following Lemma \ref{lem:linPropConstr} on the formulation for the layer propagation constraint \eqref{eq:origFormPropConstr} introduced in Section \ref{subsubsec:2_1_1} with the definition on $\lb_\neuroni^\layer$, $\ub_\neuroni^\layer$, and $\refer_\neuroni^\layer$ for $\layer \in [\layercount]$ and $\neuroni \in \neurons^\layer$.

\linpropconstr*

\begin{proof}{Proof of Lemma \ref{lem:linPropConstr}.}
    The following constraints are equivalent to \eqref{eq:origFormPropConstr}:
    \begin{align*}
        \varx_\neuroni^\layer = 1 &\Rightarrow \funca_\neuroni^\layer(\varxb^{\layer - 1}) \ge 0\\
        &\Leftrightarrow 2 \sum_{\neuronj \in \neurons^{\layer - 1}} \weight_{\neuroni \neuronj}^\layer \varx_\neuronj^{\layer - 1} \ge \sum_{\neuronj \in \neurons^{\layer - 1}} \weight_{\neuroni \neuronj}^\layer - \bias_\neuroni^\layer, \quad \forall \layer \in [\layercount], \neuroni \in \neurons^\layer,\\
        \varx_\neuroni^\layer = 0 &\Rightarrow \funca_\neuroni^\layer(\varxb^{\layer - 1}) < 0\\
        &\Leftrightarrow 2 \sum_{\neuronj \in \neurons^{\layer - 1}} \weight_{\neuroni \neuronj}^\layer \varx_\neuronj^{\layer - 1} < \sum_{\neuronj \in \neurons^{\layer - 1}} \weight_{\neuroni \neuronj}^\layer - \bias_\neuroni^\layer, \quad \forall \layer \in [\layercount], \neuroni \in \neurons^\layer.
    \end{align*}

    For $\neuroni \in \neurons^1$ and $\neuronj \in \neurons^0$, $\weight_{\neuroni \neuronj}^1 \varx_\neuronj^0$ is a multiple of $\frac{1}{\quant}$. Also, $\refer_\neuroni^1 + \frac{1}{\quant}$ is the smallest multiple of $\frac{2}{\quant}$ not smaller than $\sum_{\neuronj \in \neurons^0} \weight_{\neuroni \neuronj}^1 - \bias_\neuroni^1$, and $\refer_\neuroni^1 - \frac{1}{\quant}$ is the largest multiple of $\frac{2}{\quant}$ smaller than $\sum_{\neuronj \in \neurons^0} \weight_{\neuroni \neuronj}^1 - \bias_\neuroni^1$. Hence, the above constraints can be written as follows for $\layer = 1$:
    \begin{align*}
        \varx_\neuroni^1 = 1 &\Rightarrow 2 \sum_{\neuronj \in \neurons^0} \weight_{\neuroni \neuronj}^1 \varx_\neuronj^0 \ge \refer_\neuroni^1 + \frac{1}{\quant}, \quad \forall \neuroni \in \neurons^1,\\
        \varx_\neuroni^1 = 0 &\Rightarrow 2 \sum_{\neuronj \in \neurons^0} \weight_{\neuroni \neuronj}^1 \varx_\neuronj^0 \le \refer_\neuroni^1 - \frac{1}{\quant}, \quad \forall \neuroni \in \neurons^1.
    \end{align*}

    For $\layer \in \{2, \ldots, \layercount\}$, $\neuroni \in \neurons^\layer$, and $\neuronj \in \neurons^{\layer - 1}$, $\weight_{\neuroni \neuronj}^\layer \varx_\neuronj^{\layer - 1}$ is an integer. Also, $\refer_\neuroni^\layer + 1$ is the smallest even integer not smaller than $\sum_{\neuronj \in \neurons^{\layer - 1}} \weight_{\neuroni \neuronj}^\layer - \bias_\neuroni^\layer$, and $\refer_\neuroni^\layer - 1$ is the largest even integer smaller than $\sum_{\neuronj \in \neurons^{\layer - 1}} \weight_{\neuroni \neuronj}^\layer - \bias_\neuroni^\layer$. Hence, the above constraints can be written as follows for $\layer \in \{2, \ldots, \layercount\}$:
    \begin{align*}
        \varx_\neuroni^\layer = 1 &\Rightarrow 2 \sum_{\neuronj \in \neurons^{\layer - 1}} \weight_{\neuroni \neuronj}^\layer \varx_\neuronj^{\layer - 1} \ge \refer_\neuroni^\layer + 1, \quad \forall \neuroni \in \neurons^\layer,\\
        \varx_\neuroni^\layer = 0 &\Rightarrow 2 \sum_{\neuronj \in \neurons^{\layer - 1}} \weight_{\neuroni \neuronj}^\layer \varx_\neuronj^{\layer - 1} \le \refer_\neuroni^\layer - 1, \quad \forall \neuroni \in \neurons^\layer.
    \end{align*}

    For $\layer \in [\layercount]$, $\neuroni \in \neurons^\layer$, and $\neuronj \in \neurons^{\layer - 1}$, $\weight_{\neuroni \neuronj}^\layer \in \{-1, 0, 1\}$\ and $\varx_\neuronj^{\layer - 1} \in [0, 1]$. Hence, $\lb_\neuroni^\layer$ is a lower bound for $2 \sum_{\neuronj \in \neurons^{\layer - 1}} \weight_{\neuroni \neuronj}^\layer \varx_\neuronj^{\layer - 1}$ because it is minus two times the number of $-1$ in $\{\weight_{\neuroni \neuronj}^\layer: \neuronj \in \neurons^{\layer - 1}\}$. Also, $\ub_\neuroni^\layer$ is an upper bound for $2 \sum_{\neuronj \in \neurons^{\layer - 1}} \weight_{\neuroni \neuronj}^\layer \varx_\neuronj^{\layer - 1}$ because it is two times the number of $1$ in $\{\weight_{\neuroni \neuronj}^\layer: \neuronj \in \neurons^{\layer - 1}\}$.
    
    By exploiting this lower bound and upper bound, it can be concluded that $(\varxb^0, \varxb^1) \in \setx^0 \times \{0, 1\}^{\neuroncount^1}$ satisfies \eqref{eq:origFormPropConstr} if and only if it satisfies \eqref{eq:propInpLowerBoundConstr} and \eqref{eq:propInpUpperBoundConstr}, and $(\varxb^{\layer - 1}, \varxb^\layer) \in \{0, 1\}^{\neuroncount^{\layer - 1}} \times \{0, 1\}^{\neuroncount^\layer}$ satisfies \eqref{eq:origFormPropConstr} if and only if it satisfies \eqref{eq:propHidLowerBoundConstr} and \eqref{eq:propHidUpperBoundConstr} for $\layer \in \{2, \ldots, \layercount\}$. \Halmos
\end{proof}

\subsection{Proof of Lemma \ref{lem:oneip}}

\begin{proof}{Proof.}
    First, from a feasible solution $(\varxbh^0, \ldots, \varxbh^\layercount, \classh)$ to \eqref{eq:origForm} where $\class$ is considered as a decision variable, we obtain a feasible solution to \eqref{eq:incorpForm} with the same objective value. From $(\varxbh^0, \ldots, \varxbh^\layercount, \classh)$, $\varybh$ is defined as $\quant \varxbh^0$. Also, $\varzh_\class$ is defined as $\one_{\{\classh\}}(\class)$, and $\varvh_{\class \neuroni}$ is defined as $\varzh_\class \varxh_\neuroni^\layercount$.

    By Lemma \ref{lem:linPropConstr}, $(\varxbh^0, \ldots, \varxbh^\layercount, \varybh, \varzbh, \varvbh)$ satisfies \eqref{eq:propInpLowerBoundConstr}-\eqref{eq:propInpUpperBoundConstr} and \eqref{eq:propHidLowerBoundConstr}-\eqref{eq:propHidUpperBoundConstr} for $\layer \in \{2, \ldots, \layercount\}$. By the definition of $\varybh$, $(\varxbh^0, \ldots, \varxbh^\layercount, \varybh, \varzbh, \varvbh)$ satisfies \eqref{eq:inpQuantConstr}-\eqref{eq:inpUpperBoundConstr} and \eqref{eq:incorpFormYVar}. By the definition of $\varzbh$ and $\varvbh$, $(\varxbh^0, \ldots, \varxbh^\layercount, \varybh, \varzbh, \varvbh)$ satisfies \eqref{eq:incorpFormClassConstr}-\eqref{eq:incorpFormProdSumConstr} and \eqref{eq:incorpFormZVar}-\eqref{eq:incorpFormVVar}. Also, $(\varxbh^0, \ldots, \varxbh^\layercount, \varybh, \varzbh, \varvbh)$ satisfies \eqref{eq:incorpFormProdConstr} because $\varzbh$ and $\varxbh^\layercount$ are binary, so $(\varxbh^0, \ldots, \varxbh^\layercount, \varybh, \varzbh, \varvbh)$ is a feasible solution to \eqref{eq:incorpForm}.

    The definitions of $\varzbh$ and $\varvbh$ implies $\varzh_\class = 0$ for $\class \ne \classh$, $\varzh_{\classh} = 1$, and $\varvh_{\class \neuroni} = \varzh_\class \varxh_\neuroni^\layercount$. With this implication, the objective value of $(\varxbh^0, \ldots, \varxbh^\layercount, \classh)$ in \eqref{eq:origForm} is same as the objective value of $(\varxbh^0, \ldots, \varxbh^\layercount, \varybh, \varzbh, \varvbh)$ in \eqref{eq:incorpForm} by following steps:
    \begin{equation}
        \begin{aligned}
            \funca_{\classh}^{\layercount + 1}(\varxbh^\layercount) - \funca_{\classbar}^{\layercount + 1}(\varxbh^\layercount) &= \sum_{\class \in \neurons^{\layercount + 1} \setminus \{\classbar\}} \varzh_\class(\funca_\class^{\layercount + 1}(\varxbh^\layercount) - \funca_{\classbar}^{\layercount + 1}(\varxbh^\layercount))\\
            &= \sum_{\class \in \neurons^{\layercount + 1} \setminus \{\classbar\}} \varzh_\class\Big(\sum_{\neuroni \in \neurons^\layercount} (\weight_{\class \neuroni}^\layercount - \weight_{\classbar \neuroni}^\layercount)(2 \varxh_\neuroni^\layercount - 1) + (\bias_\class^\layercount - \bias_{\classbar}^\layercount)\Big)\\
            &= \sum_{\class \in \neurons^{\layercount + 1} \setminus \{\classbar\}} \sum_{\neuroni \in \neurons^\layercount} 2(\weight_{\class \neuroni}^\layercount - \weight_{\classbar \neuroni}^\layercount)\varzh_\class \varxh_\neuroni^\layercount\\
            &\quad + \sum_{\class \in \neurons^{\layercount + 1} \setminus \{\classbar\}} \Big(-\sum_{\neuroni \in \neurons^\layercount} (\weight_{\class \neuroni}^\layercount - \weight_{\classbar \neuroni}^\layercount) + (\bias_\class^\layercount - \bias_{\classbar}^\layercount)\Big)\varzh_\neuroni\\
            &= \sum_{\class \in \neurons^{\layercount + 1} \setminus \{\classbar\}} \sum_{\neuroni \in \neurons^\layercount} 2(\weight_{\class \neuroni}^\layercount - \weight_{\classbar \neuroni}^\layercount)\varvh_{\class \neuroni}\\
            &\quad + \sum_{\class \in \neurons^{\layercount + 1} \setminus \{\classbar\}} \Big(-\sum_{\neuroni \in \neurons^\layercount} (\weight_{\class \neuroni}^\layercount - \weight_{\classbar \neuroni}^\layercount) + (\bias_\class^\layercount - \bias_{\classbar}^\layercount)\Big)\varzh_\neuroni.
        \end{aligned}
    \label{eq:linObj}
    \end{equation}

    Next, from a feasible solution $(\varxbh^0, \ldots, \varxbh^\layercount, \varybh, \varzbh, \varvbh)$ to \eqref{eq:incorpForm}, we obtain a feasible solution to \eqref{eq:origForm} with the same objective value where $\class$ is considered as a decision variable. From $(\varxbh^0, \ldots, \varxbh^\layercount, \varybh, \varzbh, \varvbh)$, $\classh$ is defined as $\class$ satisfying $\varzh_\class = 1$, whose uniqueness is guaranteed by \eqref{eq:incorpFormClassConstr} and \eqref{eq:incorpFormZVar}.

    By \eqref{eq:propInpLowerBoundConstr}-\eqref{eq:propHidUpperBoundConstr} and Lemma \ref{lem:linPropConstr}, $(\varxbh^0, \ldots, \varxbh^\layercount, \classh)$ satisfies the layer propagation constraints \eqref{eq:origFormPropConstr}. By \eqref{eq:inpQuantConstr}-\eqref{eq:inpPertConstr}, \eqref{eq:incorpFormXInpVar}, and \eqref{eq:incorpFormYVar}, $(\varxbh^0, \ldots, \varxbh^\layercount, \classh)$ satisfies the input perturbation constraint \eqref{eq:origFormXInpVar}. Hence, $(\varxbh^0, \ldots, \varxbh^\layercount, \classh)$ is a feasible solution to \eqref{eq:origForm}.

    By \eqref{eq:linObj}, the objective value of $(\varxbh^0, \ldots, \varxbh^\layercount, \varybh, \varzbh, \varvbh)$ in \eqref{eq:incorpForm} is same as the objective value of $(\varxbh^0, \ldots, \varxbh^\layercount, \classh)$ in \eqref{eq:origForm}. The first equality holds because $\varzh_\class = 0$ for $\class \neq \classh$, and $\varzh_{\classh} = 1$. The last equality holds because $\varvh_{\class \neuroni} = \varzh_\class \varxh_\neuroni^\layercount$. For $\class \ne \classh$, $\varvh_{\class \neuroni} \ge 0$ by \eqref{eq:incorpFormVVar}, and $\varvh_{\class \neuroni} \le \varzh_\class = 0$ by \eqref{eq:incorpFormProdConstr}, so $\varvh_{\class \neuroni} = 0 = \varzh_\class \varxh_\neuroni^\layercount$. Also, $\varvh_{\classh \neuroni} = \varzh_{\classh} \varxh_\neuroni^\layercount$ by following steps:
    \begin{align*}
        \varvh_{\classh \neuroni} &= \sum_{\class \in \neurons^{\layercount + 1} \setminus \{\classbar\}} \varvh_{\class \neuroni} &(\text{For all $\class$ other than $\classh$, $\varvh_{\class \neuroni} = 0$})\\
        &= \varxh_\neuroni^\layercount &(\text{By \eqref{eq:incorpFormProdSumConstr}})\\
        &= \varzh_{\classh} \varxh_\neuroni^\layercount.
    \end{align*}

    For these reasons, \eqref{eq:incorpForm} is equivalent to \eqref{eq:origForm}. \Halmos
\end{proof}

\subsection{Proof of Convex Hull Characterization for a Single Neuron}

We prove Theorem \ref{theo:singleConvHullExt} on the characterization for the convex hull \eqref{eq:singleConvHull} for a single neuron introduced in Section \ref{subsec:2_3}, which is an extension of the main result in \cite{han2021single}. To simplify the statement of this theorem, $\layer \in \{2, \ldots, \layercount\}$ and $\neuroni \in \neurons^\layer$ are fixed, and the following notations are defined for $\neuronj \in \neurons^{\layer - 1}$ with $\lb_\neuroni^\layer$, $\ub_\neuroni^\layer$, and $\refer_\neuroni^\layer$ defined in Section \ref{subsubsec:2_1_1}:
\begin{equation*}
\begin{aligned}
    \varxb & := \varxb^{\layer - 1},\\
    \varu & := \varx_\neuroni^\layer,\\
    \weightl_\neuronj & := \weight_{\neuroni \neuronj}^\layer,\\
    \lb & := \lb_\neuroni^\layer = \sum_{\neuronj \in \neurons^{\layer - 1}} (\weight_{\neuroni \neuronj}^\layer - |\weight_{\neuroni \neuronj}^\layer|),\\
    \ub &:= \ub_\neuroni^\layer = \sum_{\neuronj \in \neurons^{\layer - 1}} (\weight_{\neuroni \neuronj}^\layer + |\weight_{\neuroni \neuronj}^\layer|),\\
    \refer &:= \refer_\neuroni^\layer = 2 \bigg\lceil\frac{\sum_{\neuronj \in \neurons^{\layer - 1}} \weight_{\neuroni \neuronj}^\layer - \bias_\neuroni^\layer}{2}\bigg\rceil - 1.
\end{aligned}
\end{equation*}
Then, this theorem can be restated as follows:

\noindent\textbf{Theorem \ref{theo:singleConvHullExt} (Restated)} {\itshape The set of $(\varxb, \varu) \in [0, 1]^{\neuroncount^{\layer - 1}} \times [0, 1]$ satisfying the following inequalities for all $\neuronjs \subset \{\neuronj \in \neurons^{\layer - 1}: \weightl_\neuronj \ne 0\}$ is \eqref{eq:singleConvHull}:
\begin{subequations}\label{eq:simpSingleConvHullIneq}
    \begin{align}
        &\sum_{\neuronj \in \neuronjs} (\weightl_\neuronj(2 \varx_\neuronj - 1) - (2 \varu - 1)) \ge (\refer - \ub + 1)\varu,\label{eq:simpSingleConvHullIneqLowerBound}\\
        &\sum_{\neuronj \in \neuronjs} (\weightl_\neuronj(2 \varx_\neuronj - 1) - (2 \varu - 1)) \le (\refer - \lb - 1)(-\varu + 1).\label{eq:simpSingleConvHullIneqUpperBound}
    \end{align}
\end{subequations}}

\begin{proof}{Proof.}
    Consider the following sets:
    \begin{equation*}
    \begin{aligned}
        \setx^+ &= \{(\varxb, 1): \varxb \in \{0, 1\}^{\neuroncount^{\layer - 1}}, \funca_\neuroni^\layer(\varxb) \ge 0\},\\
        \setx^- &= \{(\varxb, 0): \varxb \in \{0, 1\}^{\neuroncount^{\layer - 1}}, \funca_\neuroni^\layer(\varxb) < 0\}.
    \end{aligned}
    \end{equation*}
    By the proof of Lemma \ref{lem:linPropConstr}, $\setx^+$ and $\setx^-$ can be written as follows:
    \begin{equation*}
    \begin{aligned}
        \setx^+ &= \Big\{(\varxb, 1): \varxb \in \{0, 1\}^{\neuroncount^{\layer - 1}}, \sum_{\substack{\neuronj \in \neurons^{\layer - 1}\\ \weightl_\neuronj \ne 0}} \weightl_\neuronj \varx_\neuronj \ge \frac{\refer + 1}{2}\Big\},\\
        \setx^- &= \Big\{(\varxb, 0): \varxb \in \{0, 1\}^{\neuroncount^{\layer - 1}}, \sum_{\substack{\neuronj \in \neurons^{\layer - 1}\\ \weightl_\neuronj \ne 0}} \weightl_\neuronj \varx_\neuronj \le \frac{\refer - 1}{2}\Big\}.
        \end{aligned}
    \end{equation*}

    By exploiting total unimodularity of the constraints defining $\setx^+$ and $\setx^-$ arising from $\weightl_\neuronj \in \{-1, 0, 1\}$ and $\frac{\refer + 1}{2} \in \zz$, $\conv(\setx^+)$ and $\conv(\setx^-)$ can be obtained as follows:
    \begin{equation*}
\begin{aligned}
        \conv(\setx^+) &= \bigg\{(\varxb, 1): \varxb \in [0, 1]^{\neuroncount^{\layer - 1}}, \sum_{\substack{\neuronj \in \neurons^{\layer - 1}\\ \weightl_\neuronj \ne 0}} \weightl_\neuronj \varx_\neuronj \ge \frac{\refer + 1}{2}\bigg\},\\
        \conv(\setx^-)& = \bigg\{(\varxb, 0): \varxb \in [0, 1]^{\neuroncount^{\layer - 1}}, \sum_{\substack{\neuronj \in \neurons^{\layer - 1}\\ \weightl_\neuronj \ne 0}} \weightl_\neuronj \varx_\neuronj \le \frac{\refer - 1}{2}\bigg\}.
        \end{aligned}
    \end{equation*}

    These convex hulls can be used to describe \eqref{eq:singleConvHull} as follows where the last statement holds because of compactness of $\conv(\setx^+)$ and $\conv(\setx^-)$ arising from finiteness of $\setx^+$ and $\setx^-$:
    \begin{equation*}
    \begin{aligned}
        \conv(\setx^+ \cup \setx^-) &= \conv(\conv(\setx^+) \cup \conv(\setx^-))\\
        &= \cl\conv(\conv(\setx^+) \cup \conv(\setx^-)).
    \end{aligned}
    \end{equation*}

    By \cite{balas1985disjunctive}'s work on disjunctive programming, $(\varxb, \varu) \in \rr^{\neuroncount^{\layer - 1}} \times \rr$ is in \eqref{eq:singleConvHull} if and only if there exist $(\varxb^+, \varu^+), (\varxb^-, \varu^-) \in \rr^{\neuroncount^{\layer - 1}} \times \rr$ and $\lambda \in [0, 1]$ satisfying the following inequalities:
    \begin{equation*}
    \begin{aligned}
        \varxb &= \varxb^+ + \varxb^-,\\
        \varu &= \varu^+ + \varu^-,\\
        0 \le \varxb^+ &\le \lambda \cdot \ones,\\
        \varu^+ &= \lambda,\\
        2 \sum_{\substack{\neuronj \in \neurons^{\layer - 1}\\ \weightl_\neuronj \ne 0}} \weightl_\neuronj \varx_\neuronj^+ &\ge (\refer + 1)\lambda,\\
        0 \le \varxb^- &\le (1 - \lambda) \cdot \ones,\\
        \varu^- &= 0,\\
        2 \sum_{\substack{\neuronj \in \neurons^{\layer - 1}\\ \weightl_\neuronj \ne 0}} \weightl_\neuronj \varx_\neuronj^- &\le (\refer - 1)(1 - \lambda).
    \end{aligned}
    \end{equation*}
    Existence of $(\varxb^+, \varu^+), (\varxb^-, \varu^-) \in \rr^{\neuroncount^{\layer - 1}} \times \rr$ and $\lambda \in [0, 1]$ satisfying the above inequalities is identical to existence of $\varxb^+ \in \rr^{\neuroncount^{\layer - 1}}$ satisfying the following inequalities, which are provided by substituting $\varu^-$ with $0$, $\varu^+$ and $\lambda$ with $\varu$, and $\varxb^-$ with $\varxb - \varxb^+$:
    \begin{align}
        -\frac{\varu}{2} \cdot \ones \le \varxb^+ - \frac{\varu}{2} \cdot \ones &\le \frac{\varu}{2} \cdot \ones,\label{eq:plusRange}\\
        -\Big(\frac{1 - \varu}{2}\Big) \cdot \ones \le \varxb - \varxb^+ - \Big(\frac{1 - \varu}{2}\Big) \cdot \ones &\le \Big(\frac{1 - \varu}{2}\Big) \cdot \ones,\label{eq:minusRange}\\
        2 \sum_{\substack{\neuronj \in \neurons^{\layer - 1}\\ \weightl_\neuronj \ne 0}} \weightl_\neuronj \varx_\neuronj^+ &\ge (\refer + 1)\varu,\label{eq:plusConstr}\\
        2 \sum_{\substack{\neuronj \in \neurons^{\layer - 1}\\ \weightl_\neuronj \ne 0}} \weightl_\neuronj \varx_\neuronj^+ &\ge 2 \sum_{\substack{\neuronj \in \neurons^{\layer - 1}\\ \weightl_\neuronj \ne 0}} \weightl_\neuronj \varx_\neuronj - (\refer - 1)(1 - \varu).\label{eq:minusConstr}
    \end{align}
    For $\neuronj$ satisfying $\weightl_\neuronj \neq 0$, \eqref{eq:plusRange}-\eqref{eq:minusRange} are equivalent to the following inequalities because $\weightl_\neuronj \in \{-1, 1\}$:
   \begin{equation*}
    \begin{aligned}
        -\frac{\varu}{2} \cdot \ones \le \weightl_\neuronj\Big(\varxb^+ - \frac{\varu}{2}\Big) \cdot \ones &\le \frac{\varu}{2} \cdot \ones,\\
        -\frac{(1 - \varu)}{2} \cdot \ones \le \weightl_\neuronj\Big(\varxb - \varxb^+ - (\frac{1 - \varu}{2})\Big) \cdot \ones &\le \frac{(1 - \varu)}{2} \cdot \ones.
    \end{aligned}
    \end{equation*}
    By replacing $\weightl_\neuronj \varx_\neuronj^+$ with its maximum attained from the above inequalities in \eqref{eq:plusConstr}-\eqref{eq:minusConstr} and employing Fourier-Motzkin Elimination to \eqref{eq:plusRange}-\eqref{eq:minusRange} to remove $\varxb^+ - \frac{\varu}{2}$, existence of $\varxb^+ \in \rr^{\neuroncount^{\layer - 1}}$ satisfying \eqref{eq:plusRange}-\eqref{eq:minusConstr} is equivalent to whether the following inequalities are satisfied:
    \begin{equation*}
    \begin{aligned}
        \max\Big(-\frac{\varu}{2}, \varx_\neuronj + \frac{\varu}{2} - 1\Big) &\le \min\Big(\frac{\varu}{2}, \varx_\neuronj - \frac{u}{2}\Big), \quad \forall \neuronj \in \neurons^{\layer - 1},\\
        \sum_{\substack{\neuronj \in \neurons^{\layer - 1}\\ \weightl_\neuronj \ne 0}} \min((\weightl_\neuronj + 1)\varu, 2 \weightl_\neuronj \varx_\neuronj + (-\weightl_\neuronj + 1)(1 - \varu)) &\ge (\refer + 1)\varu,\\
        \sum_{\substack{\neuronj \in \neurons^{\layer - 1}\\ \weightl_\neuronj \ne 0}} \min((\weightl_\neuronj + 1)\varu, 2 \weightl_\neuronj \varx_\neuronj + (-\weightl_\neuronj + 1)(1 - \varu)) &\ge 2 \sum_{\substack{\neuronj \in \neurons^{\layer - 1}\\ \weightl_\neuronj \ne 0}} \weightl_\neuronj \varx_\neuronj - (\refer - 1)(1 - \varu).
    \end{aligned}
    \end{equation*}
    These inequalities are satisfied if and only if the following inequalities are satisfied because $|\weightl_\neuronj| = 1$ if $\weightl_\neuronj \ne 0$, $\sum_{\substack{\neuronj \in \neurons^{\layer - 1}\\ \weightl_\neuronj \ne 0}} (-\weightl_\neuronj + 1) = \sum_{\neuronj \in \neurons^{\layer - 1}} (-\weightl_\neuronj + |\weightl_\neuronj|) = -\lb$, and $\sum_{\substack{\neuronj \in \neurons^{\layer - 1}\\ \weightl_\neuronj \ne 0}} (\weightl_\neuronj + 1) = \sum_{\neuronj \in \neurons^{\layer - 1}} (\weightl_\neuronj + |\weightl_\neuronj|) = \ub$:
    \begin{gather*}
        0 \le \varx_\neuronj \le 1, \quad \forall \neuronj \in \neurons^{\layer - 1},\\
        0 \le \varu \le 1,\\
        \sum_{\substack{\neuronj \in \neurons^{\layer - 1}\\ \weightl_\neuronj \ne 0}} \min(\weightl_\neuronj(2 \varx_\neuronj - 1), 2 \varu - 1) \ge (\refer - \lb + 1)\varu - \sum_{\neuronj \in \neurons^{\layer - 1}} |\weightl_\neuronj|,\\
        \sum_{\substack{\neuronj \in \neurons^{\layer - 1}\\ \weightl_\neuronj \ne 0}} \max(\weightl_\neuronj(2 \varx_\neuronj - 1), 2 \varu - 1) \le (\ub - \refer + 1)\varu + \refer - \sum_{\neuronj \in \neurons^{\layer - 1}} \weightl_\neuronj - 1.
    \end{gather*}
    If $\weightl_\neuronj \ne 0$, $|\weightl_\neuronj| = 1$, so $|\{\neuronj \in \neurons^{\layer - 1}: \weightl_\neuronj \ne 0\}| = \sum_{\neuronj \in \neurons^{\layer - 1}} |\weightl_\neuronj|$. With this equality, by replacing $\min$ and $\max$ in the last two inequalities with linear expressions involving $\neuronjs \subset \{\neuronj \in \neurons^{\layer - 1}: \weightl_\neuronj \ne 0\}$, the set of $(\varxb, \varu)$ satisfying the last two inequalities can be written as the set of $(\varxb, \varu)$ satisfying the following linear inequalities for all $\neuronjs \subset \{\neuronj \in \neurons^{\layer - 1}: \weightl_\neuronj \ne 0\}$:
    \begin{gather*}
        \sum_{\neuronj \in \neuronjs} (\weightl_\neuronj(2 \varx_\neuronj - 1) - (2 \varu - 1)) \ge (\refer - \lb + 1)\varu - \sum_{\neuronj \in \neurons^{\layer - 1}} |\weightl_\neuronj| - \Big(\sum_{\neuronj \in \neurons^{\layer - 1}} |\weightl_\neuronj|\Big)(2 \varu - 1),\\
        \sum_{\neuronj \in \neuronjs} (\weightl_\neuronj(2 \varx_\neuronj - 1) - (2 \varu - 1)) \le (\ub - \refer + 1)\varu + \refer - \sum_{\neuronj \in \neurons^{\layer - 1}} \weightl_\neuronj - 1 - \Big(\sum_{\neuronj \in \neurons^{\layer - 1}} |\weightl_\neuronj|\Big)(2 \varu - 1).
    \end{gather*}
    
    These inequalities are equivalent to \eqref{eq:simpSingleConvHullIneqLowerBound} and \eqref{eq:simpSingleConvHullIneqUpperBound} because $\sum_{\neuronj \in \neurons^{\layer - 1}} |\weightl_\neuronj| = \frac{\ub - \lb}{2}$ and $\sum_{\neuronj \in \neurons^{\layer - 1}} (\weightl_\neuronj - |\weightl_\neuronj|) = \lb$, so \eqref{eq:singleConvHull} is the set of $(\varxb, \varu) \in [0, 1]^{\neurons^{\layer - 1}} \times [0, 1]$ satisfying \eqref{eq:simpSingleConvHullIneq} for all $\neuronjs \subset \{\neuronj \in \neurons^{\layer - 1}: \weightl_\neuronj \ne 0\}$. \Halmos
\end{proof}

\subsection{Implementation Details for Infinity-norm and 2-norm}
\label{eq:inf2details}

We explain implementation details for the case that the maximum perturbation from $\varxbbar$ is defined using an $\ell_\infty$ norm or $\ell_2$ norm. These details are about how to solve the IP subproblem \eqref{eq:varFixForm} without an IP solver in the case of $\layer = 1$ and how to implement the input perturbation constraint \eqref{eq:inpPertConstr}.

We first consider solving \eqref{eq:varFixForm} for the case of layer $\layer = 1$ and $\normp = \infty$. In this case, the optimal solution is obtained by setting $\varx_\neuronj^0$ to $\min(\varxbar_\neuronj + \pert, 1)$ for $\neuronj$ satisfying $\constc_\neuroni \weight_{\neuroni \neuronj}^1 = 1$, setting $\varx_\neuronj^0$ to $\max(\varxbar_\neuronj - \pert, 0)$ for $\neuronj$ satisfying $\constc_\neuroni \weight_{\neuroni \neuronj}^1 = -1$, and  setting $\varx_\neuronj^0$ to $\varxbar_\neuronj$ for $\neuronj$ satisfying $\constc_\neuroni \weight_{\neuroni \neuronj}^1 = 0$. All constraints defining $\setx^0$ are of the form $|\varx_\neuronj^0 - \varxbar_\neuronj| \le \pert$, so this solution is feasible, and it immediate that no other solution can yield a better objective value.

We next consider solving \eqref{eq:varFixForm} for the case of layer $\layer = 1$ and $\normp = 2$. First, observe that given a feasible solution $\varxbh^0 \in \frac{1}{\quant} \zz_+^{\neuroncount^0} \cap [0, 1]^{\neuroncount^0}$ to \eqref{eq:varFixForm}, if for some $j \in N^0$ it holds that $\constc_\neuroni \weight_{\neuroni \neuronj}^1(\varxh_\neuronj^0 - \varxbar_\neuronj) \le 0$, then another feasible solution with the same or better objective value can be obtained by replacing $\varxh_\neuronj^0$ with $\varxbar_\neuronj$. Hence, we can restrict our search for an optimal solution of \eqref{eq:varFixForm} to solutions $\varxbh^0$ that satisfy:
\begin{align*}
    \varxh_\neuronj^0 \ge \varxbar_\neuronj,  &\quad 
    \text{for $j$ such that } \constc_\neuroni \weight_{\neuroni \neuronj}^1 = 1,\\
    \varxh_\neuronj^0 \le \varxbar_\neuronj, &\quad   \text{for $j$ such that } \constc_\neuroni \weight_{\neuroni \neuronj}^1 = -1,\\
    \varxh_\neuronj^0 = \varxbar_\neuronj, &\quad   \text{for $j$ such that } \weight_{\neuroni \neuronj}^1 = 0.
\end{align*}
Let $N_0^+ = \{j \in N_0 : \weight_{\neuroni \neuronj} \neq 0\}$. Thus, we can restrict our attention to solutions defined by a vector $z \in \{0,1,\ldots,q\}^{N_0^+}$ as follows:
$$
    \varxh_\neuronj^{0}(z) = \begin{cases} \min(\varxbar_\neuronj + \frac{z_\neuronj}{\quant}, 1) & \text{if } \constc_\neuroni \weight_{\neuroni \neuronj}^1 = 1, \\ \max(\varxbar_\neuronj - \frac{z_\neuronj}{\quant}, 0) & \text{if } \constc_\neuroni \weight_{\neuroni \neuronj}^1 = -1, \\ \varxbar_\neuronj & \text{if } \weight_{\neuroni \neuronj}^1 = 0 . \end{cases}
$$
Observe that the objective in \eqref{eq:varFixForm} and the expression $\| \varxbh(z) - \varxbbar \|_2$ are monotone increasing in $z$, and hence an optimal solution will be such that increasing $z_j$ for any $j \in N_0^+$ would be infeasible. Next, for $z \in \{0,1,\ldots,q\}^{N_0^+}$, define the set
$$ N_0^S(z) = \Bigl\{ j \in N_0^+ : (c_iW_{ij}^1 = 1 \text{ and } \varxh_\neuronj^0(z) \leq 1 - 1/q ) \text{ or } (c_iW_{ij}^1 = -1 \text{ and } \varxh_\neuronj^0(z) \geq 1/q ) \Bigr\}. $$
If there exists $z \in \{0,1,\ldots,q\}^{N_0^+}$ in which $\varxbh(z)$ is feasible to \eqref{eq:varFixForm} and $N_0^S(z) = \emptyset$, then $\varxbh(z)$ is optimal to \eqref{eq:varFixForm} since this solution obtains the best possible objective. Otherwise, we claim that there exists an optimal solution $\varxbh(z)$ that satisfies:
\[ \max\{ z_j : j \in N_0^S(z)\} - \min \{z_j : j \in N_0^S(z) \} \leq 1. \]
Consider a solution that violates this condition. Another feasible solution with the same objective value can be obtained by decreasing $z_{j_1}$ by one and increasing $z_{j_2}$ by one, where $j_1 \in \arg \max\{ z_j : j \in N_0^S(z)\}$ and $j_2 \in \arg \min \{z_j : j \in N_0^S(z) \}$. Repeating this process will eventually yield a vector $z$ which satisfies this condition since there only finitely many elements achieving the max and min in the condition, and an element will necessarily be removed from one or the other as long as the max is at least two larger than the min. 

These arguments imply that an optimal solution to \eqref{eq:varFixForm} can be obtained by first finding the maximum $m$ such that the solution $\varxbh(z)$ is satisfies $\| \varxbh(z) - \varxbbar \|_2 \leq \epsilon$, where $z_j = m$ for $j \in N_0^+$. Having found this solution (e.g., by binary search), one would then iteratively select some $j \in N_0^S(z)$ and increase $z_j$, and repeat (without selecting any $j$ more than once) until no more can be increased without violating $\| \varxbh(z) - \varxbbar \|_2 \leq \epsilon$.

Finally, we describe how the constraint \eqref{eq:inpPertConstr} is formulated when using a MIP solver to solve the verification problem in the case of $p=2$. We define a decision variable $\varu_\neuronj$ to represent $|\varx_\neuronj^0 - \varxbar_\neuronj|$ for each $\neuronj \in \neurons^0$. The following inequalities are added as constraints because \eqref{eq:inpPertConstr} is satisfied if and only if there exists $\varub \in \rr_+^{\neuroncount^0}$ satisfying the following inequalities:
\begin{align*}
    \varu_\neuronj &\ge \varx_\neuronj^0 - \varxbar_\neuronj, \quad \forall \neuronj \in \neurons^0,\\
    \varu_\neuronj &\ge -\varx_\neuronj^0 + \varxbar_\neuronj, \quad \forall \neuronj \in \neurons^0,\\
    \sum_{\neuronj \in \neurons^0} \varu_\neuronj^2 &\le \pert^2.
\end{align*}
We remark that we found through preliminary empirical study that this formulation technique was computationally superior to the more direct approach of simply formulating the constraint as follows:
\[ \sum_{j \in N_0} (x_j^0 - \varxbar_\neuronj)^2 \leq \epsilon^2. \]

\subsection{Detailed Computational Results for Infinity-norm and 2-norm}
\label{app:inf2norm}

We present more detailed computational results for the case that the maximum perturbation from $\varxbbar$ is defined using an $\ell_\infty$ norm or $\ell_2$ norm. Test instances for these cases are formed with Network 5 in Table \ref{tab:net}, $\varxbbar$,  and $\pert$ that are chosen as explained in Section \ref{subsec:4_1}.

{
\SingleSpacedXI
\begin{table}
    \centering
    \begin{tabular}{c|c|S S}
        \hline
        \multirow{2}{*}{\textbf{Norm}} & \multirow{2}{*}{\begin{tabular}{c} \textbf{\# of}\\ \textbf{instances} \end{tabular}} & \textbf{Many-IP} & \textbf{1-IP}\\
        & & &\\
        \hline
        $\mathbf{\ell_\infty}$ & 47 & 143.1(16) & 78.1(6)\\
        $\mathbf{\ell_2}$ & 17 & 502.9(12) & 311.6(2)\\
        \hline
    \end{tabular}
    
    \caption{Verification time (seconds) of Many-IP and 1-IP for $\ell_\infty$ norm and $\ell_2$ norm}\label{tab:addManyIpOneIpVeriTime}
\end{table}
}
We first investigate the time required to use the IP problem \eqref{eq:incorpForm} (resp. \eqref{eq:indivForm} for all $\class$) to solve the the verification problem.
The shifted geometric mean (with a shift of 1) of times over all instances tested by both IP method Many-IP and 1-IP is computed for each method and displayed in Table \ref{tab:addManyIpOneIpVeriTime}  for the case of $\normp = \infty$ and $\normp = 2$, along with the number of test instances. For Many-IP, the number in parenthesis is the number of test instances where the time limit is hit in solving \eqref{eq:indivForm} for at least one $\class$. For 1-IP, the number in parenthesis is the number of test instances where the time limit is hit in solving \eqref{eq:incorpForm}. We find that 1-IP's verification times are smaller than Many-IP's, indicating that our technique for obtaining a linear objective also leads to faster BNN verification in the case of $\normp = \infty$ and $\normp = 2$.

{
\SingleSpacedXI
\begin{table}
    \centering
    \begin{tabular}{l|S S S S}
        \hline
        & \multicolumn{1}{c}{\textbf{1-IP}} & \multicolumn{1}{c}{\textbf{1-IP}} & \multicolumn{1}{c}{\textbf{1-IP}} & \multicolumn{1}{c}{\textbf{1-IP}}\\
        & & \multicolumn{1}{c}{\textbf{+HG}} & \multicolumn{1}{c}{\textbf{+Fix}} & \multicolumn{1}{c}{\textbf{+Fix}}\\
        & & & & \multicolumn{1}{c}{\textbf{+2Var}}\\
        \hline
        \textbf{Preprocessing time (seconds)} & 0.0 & 10.6 & 7.5 & 27.4\\
        \textbf{Verification time (seconds)} & 98.6 & 147.2 & 100.6 & 44.9\\
        \textbf{\# of solved instances (veri.)} & 38 & 35 & 37 & 44\\
        \textbf{\# of nodes (veri.)} & 1297.6 & 1283.6 & 1265.3 & 17.8\\
        \textbf{Optimization time (seconds)} & 115.5 & 178.4 & 117.5 & 57.4\\
        \textbf{\# of solved instances (opt.)} & 37 & 33 & 37 & 42\\
        \textbf{\# of nodes (opt.)} & 2182.8 & 2587.9 & 2133.7 & 86.4\\
        \textbf{LP gap} & \SI{153.4}{\percent} & \SI{156.3}{\percent} & \SI{103.4}{\percent} & \SI{35.2}{\percent}\\
        \textbf{Optimality gap} & \SI{31.8}{\percent} & \SI{46.8}{\percent} & \SI{33.6}{\percent} & \SI{11.0}{\percent}\\
        \hline
    \end{tabular}
    
    \caption{Metrics of IP methods based on 1-IP for $\ell_\infty$ norm (49 instances)}\label{tab:oneIpInfinityNorm}
\end{table}
}

Table \ref{tab:oneIpInfinityNorm} presents results analogous to those of Table \ref{tab:oneIp} in the main sections, but on instances with $\normp = \infty$. We find that in this case, neither 1-IP+HG nor 1-IP+Fix yield improvement compared to 1-IP, so generating \eqref{eq:singleConvHullIneq} and variable fixings do not improve BNN verification. Generating variable fixings decreases LP gaps, but this decrease does not lead to more effective BNN verification in the case of $\normp = \infty$. On the other hand, 1-IP+Fix+2Var yields improvement compared to the other IP methods based on 1-IP. It implies that our technique for generating layerwise derived valid inequalities, mainly two-variable inequalities, also leads to more efficient BNN verification in the case of $\normp = \infty$.

{
\SingleSpacedXI
\begin{table}
    \centering
    \begin{tabular}{l|S S S S}
        \hline
        & \multicolumn{1}{c}{\textbf{1-IP}} & \multicolumn{1}{c}{\textbf{1-IP}} & \multicolumn{1}{c}{\textbf{1-IP}} & \multicolumn{1}{c}{\textbf{1-IP}}\\
        & & \multicolumn{1}{c}{\textbf{+HG}} & \multicolumn{1}{c}{\textbf{+Fix}} & \multicolumn{1}{c}{\textbf{+Fix}}\\
        & & & & \multicolumn{1}{c}{\textbf{+2Var}}\\
        \hline
        \textbf{Preprocessing time (seconds)} & 0.0 & 35.1 & 6.7 & 9.2\\
        \textbf{Verification time (seconds)} & 1215.6 & 1305.9 & 9.3 & 10.3\\
        \textbf{\# of solved instances (veri.)} & 12 & 12 & 22 & 22\\
        \textbf{\# of nodes (veri.)} & 9844.0 & 9811.1 & 0.8 & 0.7\\
        \textbf{Optimization time (seconds)} & 1549.9 & 1632.5 & 12.2 & 12.4\\
        \textbf{\# of solved instances (opt.)} & 11 & 11 & 21 & 22\\
        \textbf{\# of nodes (opt.)} & 16647.4 & 16607.7 & 7.1 & 5.0\\
        \textbf{LP gap} & \SI{165.1}{\percent} & \SI{165.1}{\percent} & \SI{12.5}{\percent} & \SI{2.9}{\percent}\\
        \textbf{Optimality gap} & \SI{46.5}{\percent} & \SI{46.6}{\percent} & \SI{1.8}{\percent} & \SI{0.0}{\percent}\\
        \hline
    \end{tabular}
    
    \caption{Metrics of IP methods based on 1-IP for $\ell_2$ norm (22 instances)}\label{tab:oneIpTwoNorm}
\end{table}
}
In the case of $\normp = 2$, the same metrics are computed over all 22 instances tested by all IP methods based on 1-IP to compare these IP methods. Table \ref{tab:oneIpTwoNorm} shows these metrics. 1-IP+HG does not result in improvement compared to 1-IP, which means that generating \eqref{eq:singleConvHullIneq} also does not improve BNN verification in the case of $\normp = 2$. 1-IP+Fix and 1-IP+Fix+2Var yield similar improvements compared to 1-IP, which implies that generating variable fixings improves BNN verification in the case of $\normp = 2$, but on these instances there does not appear to be further improvement from two-variable inequalities.
{
\SingleSpacedXI
\begin{table}
    \centering
    \begin{tabular}{l|S S}
        \hline
        & \multicolumn{1}{c}{\textbf{1-IP}} & \multicolumn{1}{c}{\textbf{1-IP}}\\
        & \multicolumn{1}{c}{\textbf{+Fix}} & \multicolumn{1}{c}{\textbf{+Fix}}\\
        & & \multicolumn{1}{c}{\textbf{+2Var}}\\
        \hline
        \textbf{Preprocessing time (seconds)} & 7.0 & 24.0\\
        \textbf{Verification time (seconds)} & 66.3 & 41.5\\
        \textbf{\# of solved instances (veri.)} & 35 & 42\\
        \textbf{\# of nodes (veri.)} & 33.6 & 4.8\\
        \textbf{Optimization time (seconds)} & 82.3 & 82.9\\
        \textbf{\# of solved instances (opt.)} & 32 & 34\\
        \textbf{\# of nodes (opt.)} & 90.7 & 75.7\\
        \textbf{LP gap} & \SI{51.8}{\percent} & \SI{24.0}{\percent}\\
        \textbf{Optimality gap} & \SI{29.5}{\percent} & \SI{16.9}{\percent}\\
        \hline
    \end{tabular}
    
    \caption{Metrics of 1-IP+Fix and 1-IP+Fix+2Var for $\ell_2$ norm (49 instances)}\label{tab:oneIpFixTwoNorm}
\end{table}
}

To better compare 1-IP+Fix and 1-IP+Fix+2Var, we report in Table \ref{tab:oneIpFixTwoNorm} the results comparing these instances on all 49 instances solved by these two methods (many of these were excluded in Table \ref{tab:oneIpTwoNorm} since they were not solved by 1-IP or 1-IP+HG). From this table we observe that spending more time in generating more layerwise derived valid inequalities, mainly two-variable inequalities, results in reducing LP gaps and solving the verification problem more quickly on these instances. 

\subsection{Computational Study on Convex Hull Characterization for a Single Neuron}
\label{app:singneuron}

We pursue an additional computational study to validate our implementation for IP methods exploiting Theorem \ref{theo:singleConvHullExt} on convex hull characterization for a single neuron. As we see in in Table \ref{tab:oneIpLpGap}, the IP method 1-IP+HG employing the valid inequalities presented Theorem \ref{theo:singleConvHullExt} does not yield significant decreases in LP gaps compared to the IP method 1-IP. However, using these valid inequalities resulted in smaller LP gaps in \cite{han2021single}. One possible reason for this discrepancy is that \eqref{eq:singleConvHullIneq} cannot be used for $\layer = 1$ in our computational study because $\quant = 255$, but these valid inequalities can be used in the computational study of \cite{han2021single} because $\quant = 1$. To examine whether this guess is correct, an additional computational study is conducted on test instances based on ones in \cite{han2021single} to assess the impact of exploiting Theorem \ref{theo:singleConvHullExt}. In this computational study, the IP method Many-IP and Many-IP+HG, which solves the IP problem \eqref{eq:indivForm} by employing a constraint generation approach with \eqref{eq:singleConvHullIneq} as 1-IP+HG for each $\class$, are compared because $\class$ is fixed in \cite{han2021single}.

{
\SingleSpacedXI
\begin{table}
    \centering
    \begin{tabular}{c|c c c}
        \hline
        \textbf{Network} & \multicolumn{1}{c}{$\mathbf{\layercount}$} & \multicolumn{1}{c}{$\mathbf{\neuroncount^1}$} & \multicolumn{1}{c}{$\mathbf{\neuroncount^2}$}\\
        \hline
        \textbf{1} & 1 & 32 & *\\
        \textbf{2} & 1 & 64 & *\\
        \textbf{3} & 1 & 128 & *\\
        \textbf{4} & 2 & 32 & 32\\
        \textbf{5} & 2 & 64 & 64\\
        \textbf{6} & 2 & 128 & 128\\
        \hline
    \end{tabular}
    
    \caption{Networks in additional computational study}\label{tab:addNet}
\end{table}
}

Each test instance consists of a BNN, $\varxbbar$, and $\pert$. Six BNNs pretrained from \cite{han2021single} are used in this computational study. These networks are trained on the MNIST training dataset by using \cite{hubara2016binarized}'s method, which implies $\neuroncount^0 = 784$ and $\neuroncount^{\layercount + 1} = 10$. The numbers of hidden layers and the number of neurons in hidden layers of each network are reported in Table \ref{tab:addNet}. Feature vectors in the MNIST dataset are scaled to binary vectors by converting coordinates smaller than 128 to 0 and the other coordinates to 1, which indicates $\quant = 1$. Ten feature vectors are selected for $\varxbbar$. For each digit from 0 to 9, one feature vector in the MNIST test dataset whose $\classbar$ is this digit is randomly chosen. Positive integers from 1 to 5 are used as $\pert$.

{
\SingleSpacedXI
\begin{table}
    \centering
    \begin{tabular}{c|c|S S|S S}
        \hline
        \multirow{2}{*}{\textbf{Network}} & \multirow{2}{*}{$\mathbf{\pert}$} & \multicolumn{2}{c|}{\textbf{LP value}} & \multicolumn{2}{c}{\textbf{Verification time (seconds)}}\\
        \cline{3-6}
        & & \multicolumn{1}{c}{\textbf{Many-IP}} & \multicolumn{1}{c|}{\textbf{Many-IP+HG}} & \multicolumn{1}{c}{\textbf{Many-IP}} & \multicolumn{1}{c}{\textbf{Many-IP+HG}}\\
        \hline
        \multirow{5}{*}{\textbf{1}} & 1.0 & 24.5 & 6.2 & 1.2 & 1.7\\
        & 2.0 & 24.6 & 6.3 & 1.3 & 5.0\\
        & 3.0 & 24.7 & 7.9 & 1.4 & 5.8\\
        & 4.0 & 24.8 & 7.7 & 1.4 & 6.7\\
        & 5.0 & 24.8 & 7.9 & 1.5 & 8.6\\
        \hline
        \multirow{5}{*}{\textbf{2}} & 1.0 & 50.2 & 12.9 & 2.3 & 3.8\\
        & 2.0 & 50.4 & 16.9 & 4.2 & 9.9\\
        & 3.0 & 50.5 & 16.4 & 4.3 & 13.4\\
        & 4.0 & 50.7 & 18.5 & 4.3 & 16.9\\
        & 5.0 & 50.8 & 22.9 & 4.1 & 17.2\\
        \hline
        \multirow{5}{*}{\textbf{3}} & 1.0 & 80.7 & 15.9 & 4.5 & 8.6\\
        & 2.0 & 80.9 & 19.1 & 9.0 & 27.5\\
        & 3.0 & 81.2 & 20.3 & 9.2 & 34.1\\
        & 4.0 & 81.4 & 29.4 & 9.2 & 36.6\\
        & 5.0 & 81.6 & 31.8 & 9.3 & 49.3\\
        \hline
        \multirow{5}{*}{\textbf{4}} & 1.0 & 22.1 & 4.5 & 1.2 & 5.0\\
        & 2.0 & 22.1 & 11.0 & 1.4 & 21.1\\
        & 3.0 & 22.1 & 14.4 & 1.5 & 20.2\\
        & 4.0 & 22.1 & 16.6 & 1.4 & 22.6\\
        & 5.0 & 22.1 & 18.1 & 1.5 & 25.6\\
        \hline
        \multirow{5}{*}{\textbf{5}} & 1.0 & 29.9 & 12.7 & 2.5 & 8.6\\
        & 2.0 & 29.9 & 19.9 & 6.6 & 53.2\\
        & 3.0 & 29.9 & 23.6 & 6.9 & 47.5\\
        & 4.0 & 29.9 & 25.7 & 7.2 & 68.4\\
        & 5.0 & 29.9 & 27.0 & 7.5 & 84.8\\
        \hline
        \multirow{5}{*}{\textbf{6}} & 1.0 & 45.0 & 26.4 & 5.3 & 16.3\\
        & 2.0 & 45.0 & 36.2 & 28.7 & 1347.6(3)\\
        & 3.0 & 45.0 & 40.0 & 37.1 & 1312.6(3)\\
        & 4.0 & 45.0 & 42.2 & 50.4 & 1287.8(5)\\
        & 5.0 & 45.0 & 43.3 & 64.9 & 1175.9(4)\\
        \hline
    \end{tabular}
    
    \caption{Metrics of Many-IP and Many-IP+HG}\label{tab:manyIp}
\end{table}
}
For each network and $\pert$, the arithmetic mean of optimal objective values of the LP relaxation problem of \eqref{eq:indivForm} over all $\varxbbar$ and $\class$ (LP value) and the shifted geometric mean of times to solve the verification problem for all $\class$ with a shift of 1 over all $\varxbbar$ (verification time) are computed for Many-IP and Many-IP+HG to compare these IP methods. Table \ref{tab:manyIp} reports LP values and verification times. The number in parenthesis for verification time is the number of $\varxbbar$ where the time limit is hit in solving \eqref{eq:indivForm} for at least one $\class$ for each network and $\pert$. Many-IP+HG yields smaller LP values than Many-IP as in \cite{han2021single}, so our additional computational study validates our implementation for IP methods employing the valid inequalities in Theorem \ref{theo:singleConvHullExt}. In particular, the gap in LP values between Many-IP and Many-IP+HG becomes smaller for larger $\layercount$ and $\pert$ as in \cite{han2021single}. However, Many-IP+HG results in larger verification times than Many-IP, which implies that even though using the inequalities in Theorem \ref{theo:singleConvHullExt} improves LP relaxation values, it does not reduce BNN verification time.

\end{document}